\documentclass[11pt]{article}

\usepackage[margin=1in]{geometry}
\usepackage{amsmath,amssymb,amsthm}
\usepackage{booktabs}
\usepackage{graphicx}
\usepackage[hidelinks]{hyperref}
\usepackage{natbib}
\usepackage{placeins}

\IfFileExists{glyphtounicode.tex}{\input{glyphtounicode}\pdfgentounicode=1}{}

\title{Signed Evidence Flow: Conflict-Aware and Stability-Calibrated Data Analysis}

\author{
Jeffery Opoku\thanks{Corresponding author. E-mail: \texttt{jeffery.opoku01@utrgv.edu}}\\
School of Mathematical and Statistical Sciences,\\
University of Texas Rio Grande Valley, Edinburg, United States\\
\texttt{jeffery.opoku01@utrgv.edu}
\and
David Banahene\\
Robert Stempel College of Public Health and Social Work,\\
Florida International University, Miami, United States\\
\texttt{abanahene54@gmail.com}
}

\date{}

\newtheorem{definition}{Definition}
\newtheorem{theorem}{Theorem}

\newtheorem{assumption}{Assumption}
\newtheorem{proposition}{Proposition}
\newtheorem{corollary}{Corollary}

\newcommand{\E}{\mathbb{E}}
\newcommand{\Pp}{\mathbb{P}}
\newcommand{\R}{\mathbb{R}}
\newcommand{\one}{\mathbf{1}}
\newcommand{\sign}{\operatorname{sign}}

\begin{document}

\maketitle

\begin{abstract}
Modern data analysis usually gives a prediction without showing whether the evidence behind it is clear, mixed, or stable. Two cases can have the same fitted confidence while one is supported by mostly agreeing evidence and the other contains strong support and strong opposition. We propose Signed Evidence Flow (SEF), a method that turns a fitted prediction rule and a signed attribution map into an evidence audit measuring support, opposition, conflict, and perturbation stability. For complete, decision-aligned decompositions, we prove that confidence determines conflict if and only if it also determines total evidence mass, with the remaining conditional variance in closed form. A second theorem states exactly when conflict improves loss prediction beyond confidence and other audit variables. We link conflict geometrically to decision fragility through an exact flip-margin identity, and extend the construction to multi-class predictions. Across healthcare, Covertype, black-box, finance, and a ten-data-set external replication suite, cross-fitted analysis shows that conflict adds error-ranking information beyond confidence and attribution entropy on several tasks, with gains reaching $+0.144$ AUC on Magic Telescope and $+0.091$ on Credit Card Default. The direction is not universal: in some tasks, low-conflict confident cases are riskier. We introduce ScopeGate, a held-out permutation diagnostic that checks direction and calibrates the review rule before deployment. SEF reveals evidence structure; ScopeGate determines whether that structure supports safe review triage in a given population.
\end{abstract}


\section{Introduction}

Many data-analysis tools answer one question well:

\begin{quote}
What is the prediction?
\end{quote}

In practice, this is not enough. A doctor, economist, policy analyst, data scientist, or reviewer may also need to know whether the evidence behind the prediction is clear or mixed. Two people can receive the same predicted risk score. For one person, most of the evidence may point in the same direction. For the other, strong evidence may push both for and against the prediction. Treating these two cases as equally trustworthy can be risky.

Existing explanation tools have made major progress. Shapley-based explanations, local surrogate models, permutation scores, and gradient-based attributions help describe how features contribute to a prediction \citep{shapley1953value,lundberg2017unified,ribeiro2016why,sundararajan2017axiomatic}. Most of these tools, however, focus on contribution size. They can show that one variable pushed the prediction up and another pushed it down, but they usually do not treat that disagreement as the main object of study.

This paper proposes Signed Evidence Flow (SEF). The method starts with any fitted model and a reference distribution. It then breaks the prediction into signed evidence terms. Positive evidence pushes the prediction toward a target outcome. Negative evidence pushes it away. SEF measures the balance between the two. If most evidence points in the same direction, the prediction has low conflict. If support and opposition are both large, the prediction has high conflict.

The guiding question is:

\begin{quote}
Among predictions that already look confident, does the evidence structure still separate safer cases from riskier ones?
\end{quote}

The working idea is simple: a prediction is often more reliable when its evidence is directional, stable, and low-conflict. The word ``often'' matters. Later experiments show that this direction can reverse in some finance tasks, so the paper treats the direction as something to test rather than assume.

SEF therefore reports five quantities (Table~\ref{tab:sef-notation}):

\begin{enumerate}
\item signed feature or group evidence;
\item total support and total opposition;
\item an evidence conflict score;
\item a perturbation stability score;
\item a reliable evidence score.
\end{enumerate}

\begin{table}[!htbp]
\centering
\caption{Main SEF quantities. The table is included to keep the notation stable across the paper.}
\label{tab:sef-notation}
\small
\begin{tabular}{p{0.24\textwidth}p{0.22\textwidth}p{0.44\textwidth}}
\toprule
Quantity & Symbol & Plain meaning \\
\midrule
Support & $S^+(x)$ & Evidence pushing toward the reported outcome \\
Opposition & $S^-(x)$ & Evidence pushing against the reported outcome \\
Evidence mass & $M(x)$ & Total amount of signed evidence, ignoring sign \\
Net evidence & $N(x)$ & Support minus opposition \\
Conflict & $C(x)$ & How much support and opposition coexist \\
Stability & $\operatorname{Stab}(x)$ & How stable the evidence is under perturbation \\
Reliable evidence score & $\operatorname{RES}(x)$ & Directional, stable, low-conflict evidence \\
SEF risk & $R_{\mathrm{SEF}}(x)$ & One minus the reliable evidence score \\
\bottomrule
\end{tabular}
\end{table}

SEF is not meant to replace prediction models. It sits on top of them. A user can fit logistic regression, a random forest, gradient boosting, a neural network, or another predictor, and then use SEF to study the evidence behind individual predictions.

\subsection{Contributions}

This paper makes three main contributions.

First, we define a signed evidence-flow decomposition for model outputs and prove three exact results: an if-and-only-if identification theorem stating when confidence determines conflict, a closed-form conditional variance for the conflict variation left within a confidence level, and an exact risk decomposition for when conflict adds loss-prediction value beyond existing audit variables.

Second, we make evidence conflict a geometric and statistical object. We prove that conflict is a normalized $\ell_1$ distance from one-sided evidence, derive an exact flip-margin identity linking conflict to decision fragility, establish perturbation radii for stable evidence decisions, and extend the full construction to multi-class predictions through pairwise class gaps. We then introduce ScopeGate: a positive screening proposition giving the minimal sufficient condition for review improvement, paired with a finite-sample permutation diagnostic that checks direction before deployment.

Third, we evaluate SEF under repeated splits on healthcare, finance, Covertype, multi-class, non-linear black-box, and independent external replication tasks spanning 22 data sets. The key comparison uses cross-fitted error-risk models to ask whether conflict adds information after confidence and attribution entropy are already known. Perturbation stability, concentration bounds, and split-conformal screening are supporting tools rather than separate novelty claims.

\section{Setup}

Let $X=(X_1,\ldots,X_p)$ be a feature vector and let $Y$ be an outcome. Suppose a model has already been trained and gives a real-valued score
\[
f:\mathcal{X}\rightarrow \R.
\]
For binary classification, $f(x)$ may be the logit score
\[
f(x)=\log\left\{\frac{\widehat{P}(Y=1\mid X=x)}
{\widehat{P}(Y=0\mid X=x)}\right\}.
\]
For regression, $f(x)$ may be the predicted mean. For risk scoring, $f(x)$ may be a predicted risk or transformed risk.

Let $P_0$ be a reference distribution. This may be the training distribution, a chosen baseline population, a subgroup, or a scientific reference profile. Define the baseline score
\[
\mu_0=\E_{P_0}\{f(X)\}.
\]
SEF studies the contrast
\[
\Delta f(x)=f(x)-\mu_0.
\]
This is the amount by which the model score at $x$ differs from the reference score.

\section{Mathematical Scope}

Before defining SEF, we state the scope of the results. This keeps the math clear and prevents the claims from being read too broadly.

First, the algebraic identities in this paper are pointwise statements. They hold for a fixed fitted model, a fixed reference distribution, and a fixed point $x$ once the signed evidence terms have been chosen.

Second, the conflict and flip-margin results assume that the signed evidence terms are the quantities being audited. They do not require the evidence construction to be unique. Different attribution rules may produce different evidence terms, and SEF should be reported together with the chosen attribution rule.

Third, the stability bounds treat the perturbation draws as independent conditional on the fitted procedure and the point being audited. This is the usual setting for bootstrap-style Monte Carlo estimates. If perturbations are dependent, the same formulas should be viewed as diagnostic approximations rather than exact concentration bounds.

Fourth, the conformal reliability statement uses the standard split-conformal assumption: the calibration points and the future test point are exchangeable for the chosen unreliability score. Under distribution shift, the guarantee should be recalibrated or interpreted with care.

Finally, SEF is not a causal method by itself. It studies the evidence used by a fitted prediction rule. A signed evidence term should not be read as a causal effect unless the analyst adds separate causal assumptions.

\section{Signed Evidence Decomposition}

The first step is to decompose $\Delta f(x)$ into signed evidence terms.

\begin{definition}[Signed evidence decomposition]
A signed evidence decomposition is a collection of functions
\[
E_1(x),\ldots,E_p(x)
\]
such that
\[
\Delta f(x)=\sum_{j=1}^p E_j(x)+r(x),
\]
where $r(x)$ is a residual term. When $r(x)=0$, the decomposition is complete.
\end{definition}

The term $E_j(x)$ is the signed evidence assigned to feature or feature group $j$. If $E_j(x)>0$, that feature pushes the score above baseline. If $E_j(x)<0$, it pushes the score below baseline. If $E_j(x)=0$, it gives no net evidence relative to the reference.

One possible construction is the Shapley-style conditional contrast
\[
E_j(x)=
\sum_{S\subseteq \{1,\ldots,p\}\setminus \{j\}}
w(S)
\left[
\E\{f(X)\mid X_S=x_S,X_j=x_j\}
-
\E\{f(X)\mid X_S=x_S\}
\right],
\]
where
\[
w(S)=\frac{|S|!(p-|S|-1)!}{p!}.
\]
Other decompositions may also be used, provided they return signed feature-level or group-level evidence terms.

\subsection{A Model-Agnostic Perturbation Construction}

For implementation, SEF can use a simple perturbation contrast that works with any fitted model. Let $X_j^0$ be a draw from the reference distribution of feature $j$, independent of the observed point $x$. Define
\[
\widetilde{E}_j(x)
=
f(x)-
\E_{P_0}\{f(x_{-j},X_j^0)\}.
\]
This measures how much the model score changes when feature $j$ is replaced by a reference value while the other features are held at the observed values. Positive values mean the observed feature value raises the score relative to its reference behavior. Negative values mean it lowers the score.

Because perturbation contrasts do not always add exactly to $\Delta f(x)$, SEF can use the residual-spread version
\[
E_j(x)
=
\widetilde{E}_j(x)
+
\frac{\Delta f(x)-\sum_{\ell=1}^p \widetilde{E}_\ell(x)}
{p},
\]
when a complete decomposition is required. When exact additivity is not needed, the raw contrasts $\widetilde{E}_j(x)$ can be used directly for support, opposition, conflict, and stability scoring.

This construction is deliberately practical. It makes SEF usable with logistic regression, random forests, gradient boosting, neural networks, or any other model that can be evaluated after feature replacement.

\section{Support, Opposition, and Conflict}

Given signed evidence terms, split each term into positive and negative parts:
\[
E_j^+(x)=\max\{E_j(x),0\},
\qquad
E_j^-(x)=\max\{-E_j(x),0\}.
\]

\begin{definition}[Support and opposition]
The total supporting evidence and total opposing evidence are
\[
S^+(x)=\sum_{j=1}^p E_j^+(x),
\qquad
S^-(x)=\sum_{j=1}^p E_j^-(x).
\]
\end{definition}

The net evidence is
\[
N(x)=S^+(x)-S^-(x),
\]
and the total evidence mass is
\[
M(x)=S^+(x)+S^-(x).
\]

\begin{definition}[Decision-aligned evidence]
Let $\tau$ be the decision threshold on the model-score scale. Evidence is
\emph{decision-aligned} when the threshold-centered score has the complete
decomposition
\[
f(x)-\tau=\sum_{j=1}^p E_j(x)=N(x).
\]
If the feature-level evidence instead sums to $f(x)-\mu_0$, decision alignment
is obtained by adding the fixed offset $\mu_0-\tau$ as one more evidence
coordinate. All results that connect conflict to model confidence or decision
flips use decision-aligned evidence. The purely algebraic support, opposition,
and distance identities do not need this extra convention.
\end{definition}

This distinction matters. A reference-centered decomposition explains movement
away from a reference prediction. A decision-aligned decomposition explains
movement relative to the actual decision threshold. They coincide when
$\mu_0=\tau$.

\begin{definition}[Evidence direction index]
For a small $\varepsilon>0$, define
\[
D(x)=\frac{S^+(x)-S^-(x)}
{S^+(x)+S^-(x)+\varepsilon}.
\]
\end{definition}

The direction index lies near $1$ when evidence strongly supports the positive direction, near $-1$ when evidence strongly supports the negative direction, and near $0$ when evidence is balanced.

\begin{definition}[Evidence conflict score]
The evidence conflict score is
\[
C(x)=
\frac{2\min\{S^+(x),S^-(x)\}}
{S^+(x)+S^-(x)+\varepsilon}.
\]
\end{definition}

The conflict score is near $0$ when nearly all evidence points the same way. It is near $1$ when support and opposition are balanced.

\begin{proposition}[Conflict-direction identity]
If $\varepsilon=0$ and $M(x)>0$, then
\[
C(x)=1-|D(x)|.
\]
\end{proposition}

\begin{proof}
Let $a=S^+(x)$ and $b=S^-(x)$. Since
\[
a+b=|a-b|+2\min\{a,b\},
\]
we have
\[
\frac{2\min\{a,b\}}{a+b}
=1-\frac{|a-b|}{a+b}.
\]
The left side is $C(x)$ and the second term on the right is $|D(x)|$.
\end{proof}

This identity makes conflict easy to interpret. Conflict is the part of total evidence mass that does not create clear direction.

\begin{corollary}[Hidden evidence mass at a fixed score]
If $\varepsilon=0$, $M(x)>0$, and $N(x)\neq 0$, then
\[
M(x)=\frac{|N(x)|}{1-C(x)}.
\]
\end{corollary}

\begin{proof}
From the conflict-direction identity, $1-C(x)=|D(x)|=|N(x)|/M(x)$. Rearranging gives the result.
\end{proof}

This corollary explains why confidence alone can miss something important. In many classifiers, confidence is mainly tied to the size of a score or margin. Two cases can have similar scores, while one case has much more total evidence pushing in both directions. SEF exposes that hidden evidence mass through conflict.

\begin{proposition}[Conflict is not identified by a net-score magnitude]
Assume $\varepsilon=0$. Fix a nonzero net evidence value $n$ and any evidence mass $m\geq |n|$. There are nonnegative support and opposition values with $S^+-S^-=n$ and $S^++S^-=m$, and their conflict is
\[
C=1-\frac{|n|}{m}.
\]
Therefore a score magnitude or confidence value that determines $|n|$ does not determine conflict unless evidence mass is itself determined by that score magnitude.
\end{proposition}

\begin{proof}
Set
\[
S^+=\frac{m+n}{2},
\qquad
S^-=\frac{m-n}{2}.
\]
The condition $m\geq |n|$ makes both quantities nonnegative. Their difference is $n$, their sum is $m$, and the conflict-direction identity gives $C=1-|n|/m$.
\end{proof}

This proposition resolves a basic identification question. Conflict and a
score magnitude may be computed from the same features and fitted parameters,
but they compress those inputs differently. The score magnitude records the
size of the net evidence. Conflict also depends on how much evidence cancels
before that net value is reached. The proposition does not prove that conflict
predicts error. It proves only that conflict can contain information that is
absent from the net-score magnitude. The next theorem states when this becomes
a statement about reported confidence.

\begin{theorem}[General and decision-aligned identification boundary]
Treat $N$, $M$, and $C$ as random variables over cases. Assume
$\varepsilon=0$ and $M>0$ almost surely, and let $Q$ be any reported
confidence score.

\begin{enumerate}
\item Conflict is determined by confidence if and only if the normalized net
evidence $|N|/M$ is determined by confidence. Moreover,
\[
\operatorname{Var}(C\mid Q)
=
\operatorname{Var}\!\left(\frac{|N|}{M}\,\middle|\,Q\right).
\]
\item Suppose in addition that $N\neq 0$ almost surely and $|N|$ is determined
by $Q$. Then conflict is determined by confidence if and only if evidence mass
is determined by confidence, and
\[
\operatorname{Var}(C\mid Q)
=
|N|^2\operatorname{Var}\!\left(\frac{1}{M}\,\middle|\,Q\right).
\]
\end{enumerate}
\end{theorem}

\begin{proof}
The conflict identity gives
\[
C=1-\frac{|N|}{M}.
\]
Therefore $C$ is a function of $Q$ if and only if $|N|/M$ is a function of
$Q$, and subtracting from one does not change conditional variance. This proves
the first claim.

For the second claim, write $|N|=a(Q)$ for some measurable function $a$. If
$M=h(Q)$, then $C=1-a(Q)/h(Q)$ is determined by $Q$. Conversely, if
$C=g(Q)$, then
\[
M=\frac{a(Q)}{1-g(Q)}
\]
is determined by $Q$. The denominator is positive because $N\neq0$ implies
$C<1$. Finally, conditional on $Q$, the value $|N|$ is fixed, so
\[
\operatorname{Var}\!\left(\frac{|N|}{M}\,\middle|\,Q\right)
=
|N|^2\operatorname{Var}\!\left(\frac{1}{M}\,\middle|\,Q\right).
\]
\end{proof}

\begin{corollary}[Binary logistic confidence]
For a binary logistic classifier with decision-aligned evidence,
\[
Q=\{1+\exp(-|N|)\}^{-1}
\]
is one-to-one in $|N|$. Therefore conflict is determined by reported
confidence if and only if evidence mass is determined by reported confidence.
\end{corollary}

\begin{proof}
The logistic confidence map is strictly increasing on $|N|\geq0$, so $|N|$ is
determined by $Q$. Apply the second part of the theorem.
\end{proof}

The theorem gives a sharp answer to the identification objection without
requiring a special baseline. For any reported confidence, it identifies the
unresolved quantity as normalized net evidence. For a decision-aligned
logistic decomposition, this simplifies further: conflict can differ within a
confidence level only when total evidence mass differs within that level. The
conditional variance measures exactly how much evidence structure remains
hidden after confidence is known.

\begin{theorem}[Exact conditional value for error-risk prediction]
Let $L$ be any square-integrable loss, such as the binary indicator that a classifier is wrong. Let $Z$ collect the information already used by an audit model, for example confidence and attribution entropy. Define
\[
\mu_0(Z)=\E(L\mid Z),
\qquad
\mu_1(Z,C)=\E(L\mid Z,C).
\]
Under squared-error risk, the best possible improvement from adding conflict is
\[
\E\{(L-\mu_0)^2\}
-
\E\{(L-\mu_1)^2\}
=
\E\{(\mu_1-\mu_0)^2\}
=
\E\!\left[\operatorname{Var}\{\mu_1(Z,C)\mid Z\}\right].
\]
The improvement is positive if and only if mean loss changes with conflict after $Z$ is known.
\end{theorem}

\begin{proof}
Write
\[
L-\mu_0=(L-\mu_1)+(\mu_1-\mu_0).
\]
After squaring and taking expectations, the cross term is zero because
\[
\E\{L-\mu_1\mid Z,C\}=0.
\]
This gives the first equality. Also,
\[
\E(\mu_1\mid Z)=\E\{\E(L\mid Z,C)\mid Z\}=\E(L\mid Z)=\mu_0,
\]
so the second equality is the conditional-variance identity.
\end{proof}

This result separates two claims that are easy to mix up. The first theorem
says when conflict contains structure not identified by confidence under a
decision-aligned decomposition. The second says when that structure is useful
for predicting loss. Variation in evidence mass is necessary for conflict to
vary within confidence under these assumptions, but it is not enough to
guarantee better error ranking. The loss must also change with that remaining
variation. This is why the paper tests incremental value and review direction
separately.

\subsection{Identification stress test}

We checked both the positive and negative sides of these results in a controlled simulation. Confidence takes twelve discrete levels, so the base error-risk model can estimate a separate error rate at every confidence level. In the identified-mass scenario, evidence mass is fixed within each level. Conflict is then a deterministic function of confidence, and error depends only on confidence. In the hidden-mass scenario, evidence mass varies within each level and error also depends on conflict. We repeat each scenario 50 times with 6,000 cases and use five-fold cross-fitting.

\begin{figure}[!htbp]
\centering
\includegraphics[width=0.96\textwidth]{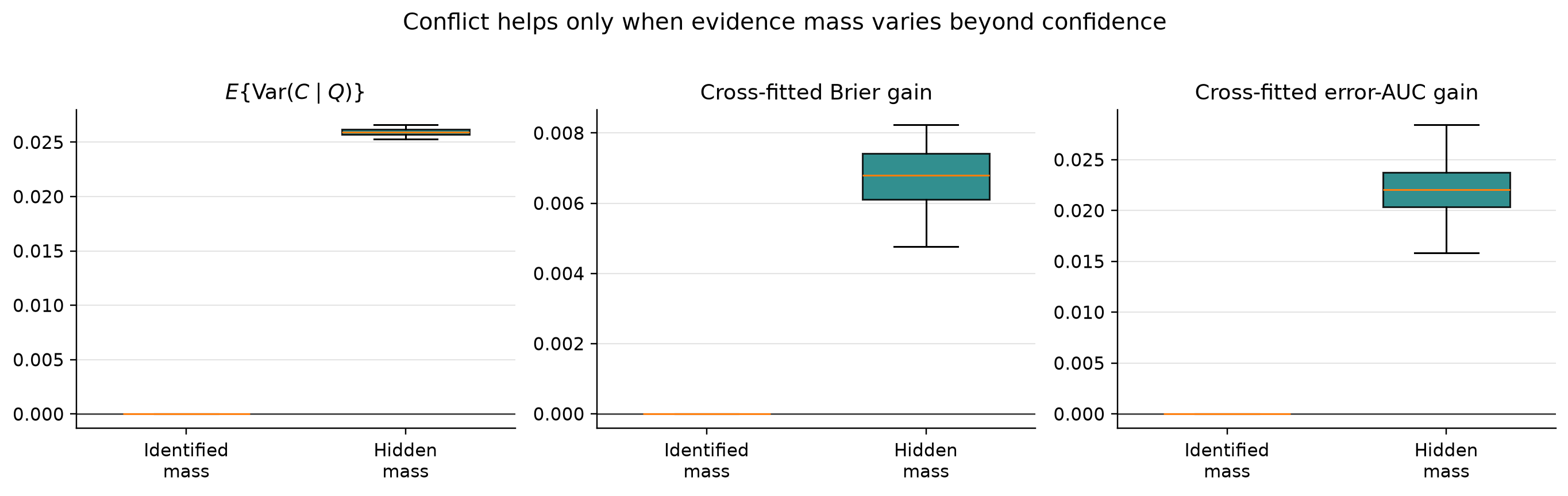}
\caption{Stress test of the identification boundary. When evidence mass is determined by confidence, conditional conflict variance is zero and adding conflict gives essentially no cross-fitted Brier or error-AUC gain. When evidence mass varies within confidence levels and loss changes with conflict, both gains become positive.}
\label{fig:identification-stress}
\end{figure}

Figure~\ref{fig:identification-stress} shows both scenarios. The negative control behaves as the theorem predicts. Its average conditional conflict variance is exactly $0$ up to numerical precision. Its mean cross-fitted Brier gain is $-0.0000003$, and its mean error-AUC change is $-0.000002$. In the hidden-mass scenario, average conditional conflict variance is $0.0260$, mean Brier gain is $0.00675$, and mean error-AUC gain is $0.0221$. Brier gain is positive in all 50 repetitions. The purpose is not to claim these synthetic effect sizes are universal. It is to show that SEF gives no artificial gain when confidence already identifies conflict, while recovering conditional value when hidden evidence mass truly affects error risk.

\begin{theorem}[Conflict as distance from one-sided evidence]
Let $e(x)=(E_1(x),\ldots,E_p(x))$ and assume $M(x)>0$ and $\varepsilon=0$. Define the two one-sided evidence cones
\[
\mathcal{K}_+=\{u\in\R^p:u_j\geq 0\ \text{for all }j\},
\qquad
\mathcal{K}_-=\{u\in\R^p:u_j\leq 0\ \text{for all }j\}.
\]
Then
\[
C(x)
=
\frac{2}{M(x)}
\min\left\{
\inf_{u\in\mathcal{K}_+}\|e(x)-u\|_1,
\inf_{u\in\mathcal{K}_-}\|e(x)-u\|_1
\right\}.
\]
\end{theorem}

\begin{proof}
The closest point to $e(x)$ in $\mathcal{K}_+$ is obtained by replacing every negative evidence term by zero and leaving every nonnegative term unchanged. The resulting $\ell_1$ distance is $S^-(x)$. Similarly, the closest point to $e(x)$ in $\mathcal{K}_-$ has distance $S^+(x)$. The smaller of the two distances is therefore $\min\{S^+(x),S^-(x)\}$. Dividing twice this distance by $M(x)$ gives $C(x)$.
\end{proof}

This result gives conflict a geometric meaning. Low conflict means the signed evidence vector is close to one of the two one-sided cones. High conflict means the evidence vector is far from both cones, so the prediction is built from substantial evidence on both sides.

\section{Decision Fragility and Flip Margin}

Suppose the evidence is decision-aligned, so the decision is based on the sign
of net evidence:
\[
\widehat{y}(x)=\one\{N(x)\geq 0\}.
\]
If evidence is highly conflicted, a small change in signed evidence can flip the decision.

\begin{definition}[Evidence flip margin]
The evidence flip margin is the smallest total amount of signed evidence that must move from the winning side to the losing side in order to make the net evidence zero. When $M(x)>0$ and $\varepsilon=0$, this margin is
\[
m_{\mathrm{flip}}(x)=\frac{|S^+(x)-S^-(x)|}{2}.
\]
Here ``move'' means transferring the same amount from one side to the other. A transfer of size $u$ lowers the winning total by $u$ and raises the losing total by $u$, so it closes the evidence gap by $2u$.
\end{definition}

\begin{theorem}[Conflict controls flip margin]
Assume $M(x)>0$ and $\varepsilon=0$. The evidence flip margin is
\[
m_{\mathrm{flip}}(x)
=\frac{|S^+(x)-S^-(x)|}{2}
=\frac{M(x)\{1-C(x)\}}{2}.
\]
\end{theorem}

\begin{proof}
If $S^+(x)\geq S^-(x)$, the positive side wins by gap $S^+(x)-S^-(x)$. Moving an amount $u$ of evidence from the positive side to the negative side reduces the gap by $2u$. The smallest value that closes the gap is
\[
u=\frac{S^+(x)-S^-(x)}{2}.
\]
The case $S^-(x)>S^+(x)$ is the same with signs reversed. Therefore
\[
m_{\mathrm{flip}}(x)=\frac{|S^+(x)-S^-(x)|}{2}.
\]
Using $|S^+-S^-|=M(1-C)$ from the previous identity gives the second equality.
\end{proof}

This theorem gives SEF its main mathematical message:

\begin{quote}
High conflict means low evidence margin, and low evidence margin means the decision is easier to flip.
\end{quote}

This statement is related in spirit to margin-based learning, where larger margins often correspond to more stable or more reliable decisions \citep{schapire1998boosting}. The difference is that the SEF margin is not a geometric classifier margin. It is an evidence margin computed from signed support and opposition inside one prediction.

\begin{theorem}[Perturbation radius for a stable evidence decision]
Let $e'(x)$ be another signed evidence vector for the same case, with net evidence $N'(x)=\sum_{j=1}^p e'_j(x)$. If
\[
\|e'(x)-e(x)\|_1 < |N(x)|,
\]
then $N'(x)$ has the same sign as $N(x)$. Equivalently, using the SEF conflict score,
\[
\|e'(x)-e(x)\|_1 < M(x)\{1-C(x)\}
\]
is enough to prevent a sign flip.
\end{theorem}

\begin{proof}
By the triangle inequality,
\[
|N'(x)-N(x)|
=
\left|\sum_{j=1}^p \{e'_j(x)-e_j(x)\}\right|
\leq
\sum_{j=1}^p |e'_j(x)-e_j(x)|
=
\|e'(x)-e(x)\|_1.
\]
If this quantity is smaller than $|N(x)|$, the perturbed net evidence cannot cross zero, so the decision sign cannot change. Since $|N(x)|=M(x)\{1-C(x)\}$ when $\varepsilon=0$, the second statement follows.
\end{proof}

This theorem is useful because it turns conflict into a robustness radius. At fixed evidence mass, lower conflict gives a larger radius of harmless evidence perturbations. Higher conflict gives a smaller radius, so less perturbation is needed before the decision can change.

\section{Multi-Class Signed Evidence Flow}

The binary construction extends naturally to multi-class prediction. Suppose a fitted model gives class scores $f_1(x),\ldots,f_K(x)$ and predicts
\[
\widehat{k}(x)=\arg\max_{1\leq k\leq K} f_k(x).
\]
For the predicted class $\widehat{k}$ and a competing class $\ell\neq\widehat{k}$, define the pairwise score gap
\[
G_{\widehat{k},\ell}(x)=f_{\widehat{k}}(x)-f_\ell(x).
\]
Assume this gap has signed evidence terms
\[
G_{\widehat{k},\ell}(x)
=
b_{\widehat{k},\ell}
+
\sum_{j=1}^p E_j^{(\widehat{k},\ell)}(x),
\]
where the baseline term may be treated as an additional evidence coordinate. Applying the binary SEF definitions to each pair gives pairwise evidence mass $M_{\widehat{k},\ell}(x)$ and pairwise conflict $C_{\widehat{k},\ell}(x)$.

\begin{definition}[Multi-class conflict and robustness radius]
Define
\[
C_{\mathrm{multi}}(x)
=
\max_{\ell\neq\widehat{k}}
C_{\widehat{k},\ell}(x)
\]
and
\[
\rho_{\mathrm{multi}}(x)
=
\min_{\ell\neq\widehat{k}}
M_{\widehat{k},\ell}(x)
\{1-C_{\widehat{k},\ell}(x)\}.
\]
Thus, multi-class conflict records the most conflicted comparison with a rival class, while the robustness radius records the smallest pairwise evidence gap.
\end{definition}

\begin{theorem}[Multi-class evidence robustness]
Suppose every pairwise evidence gap for the predicted class is positive. Let $e^{(\widehat{k},\ell)}(x)$ and $e'^{(\widehat{k},\ell)}(x)$ be the original and perturbed evidence vectors for competitor $\ell$. If
\[
\max_{\ell\neq\widehat{k}}
\left\|
e'^{(\widehat{k},\ell)}(x)
-
e^{(\widehat{k},\ell)}(x)
\right\|_1
<
\rho_{\mathrm{multi}}(x),
\]
then the predicted class remains $\widehat{k}$.
\end{theorem}

\begin{proof}
For every competitor $\ell$, the binary perturbation theorem shows that the pairwise gap remains positive whenever the pairwise evidence perturbation is smaller than
\[
M_{\widehat{k},\ell}(x)
\{1-C_{\widehat{k},\ell}(x)\}.
\]
The assumed bound is smaller than the minimum of these pairwise radii. Therefore $f_{\widehat{k}}(x)>f_\ell(x)$ remains true for every $\ell\neq\widehat{k}$, so the predicted class cannot change.
\end{proof}

This extension keeps the original interpretation. A multi-class decision is fragile when at least one important class comparison is built from strongly competing evidence or has a small pairwise evidence gap.

\section{Perturbation Stability}

Large evidence is not enough. A feature may look important only because of sampling noise, missingness, a fragile preprocessing choice, or a small shift in the data. SEF therefore measures whether evidence survives perturbation.

Let $T_b$ be a perturbation operator. Examples include bootstrap resampling, feature-noise perturbation, missingness perturbation, subgroup perturbation, or model refitting. Let
\[
E_j^{(b)}(x)
\]
be the evidence term recomputed under perturbation $b$.

\begin{definition}[Sign stability]
The sign stability of feature $j$ at $x$ is
\[
\operatorname{Stab}_j(x)
=
\Pp_b\left\{
\sign(E_j^{(b)}(x))=\sign(E_j(x))
\right\}.
\]
\end{definition}

When $\operatorname{Stab}_j(x)$ is high, feature $j$ keeps the same evidence direction under perturbation. When it is low, the feature's role is unstable.

\begin{definition}[Evidence-weighted stability]
The overall evidence stability is
\[
\operatorname{Stab}(x)
=
\frac{\sum_{j=1}^p |E_j(x)|\operatorname{Stab}_j(x)}
{\sum_{j=1}^p |E_j(x)|+\varepsilon}.
\]
\end{definition}

This definition gives more weight to features that carry more evidence.

\section{Finite-Sample Stability Bounds}

In practice, stability is estimated using $B$ perturbations:
\[
\widehat{\operatorname{Stab}}_j(x)
=
\frac{1}{B}
\sum_{b=1}^B
\one\{\sign(E_j^{(b)}(x))=\sign(E_j(x))\}.
\]

\begin{theorem}[Feature-level stability concentration]
For fixed $x$ and feature $j$, suppose the $B$ perturbation indicators are independent conditional on the fitted procedure and the audited point. Then, for any $t>0$,
\[
\Pp\left(
\left|
\widehat{\operatorname{Stab}}_j(x)
-
\operatorname{Stab}_j(x)
\right|>t
\right)
\leq
2\exp(-2Bt^2).
\]
\end{theorem}

\begin{proof}
For fixed $x$ and $j$, each agreement indicator is a Bernoulli random variable with mean $\operatorname{Stab}_j(x)$. The result follows from Hoeffding's inequality.
\end{proof}

\begin{theorem}[Uniform stability concentration]
For fixed $x$, suppose the perturbation indicators are independent conditional on the fitted procedure and the audited point. Then, with probability at least $1-\delta$,
\[
\max_{1\leq j\leq p}
\left|
\widehat{\operatorname{Stab}}_j(x)
-
\operatorname{Stab}_j(x)
\right|
\leq
\sqrt{\frac{\log(2p/\delta)}{2B}}.
\]
\end{theorem}

\begin{proof}
Apply the previous theorem to each feature and take a union bound over $j=1,\ldots,p$.
\end{proof}

The Hoeffding bound is simple but does not use the fact that sign-agreement indicators often have low variance. A variance-sensitive bound can be tighter when a feature is either very stable or very unstable.

\begin{theorem}[Variance-sensitive stability concentration]
For fixed $x$ and feature $j$, let
\[
v_j(x)=\operatorname{Stab}_j(x)\{1-\operatorname{Stab}_j(x)\}.
\]
Under the same conditional independence assumption, with probability at least $1-\delta$,
\[
\left|
\widehat{\operatorname{Stab}}_j(x)
-
\operatorname{Stab}_j(x)
\right|
\leq
\sqrt{
\frac{2v_j(x)\log(2/\delta)}{B}
}
+
\frac{2\log(2/\delta)}{3B}.
\]
The same statement holds uniformly over $p$ features after replacing $\delta$ by $\delta/p$.
\end{theorem}

\begin{proof}
Each sign-agreement indicator is Bernoulli with variance $v_j(x)$ and absolute centered value at most one. The first claim follows from the two-sided Bernstein inequality. Applying it with failure probability $\delta/p$ for each feature and taking a union bound gives the uniform statement \citep{boucheron2013concentration}.
\end{proof}

When stability is near zero or one, $v_j(x)$ is small and the leading term can be noticeably smaller than the Hoeffding radius. The bound still depends on the unknown population variance, so in practice it is best used as a theoretical guide or with a valid empirical-variance upper bound.

\section{Reliable Evidence Score}

SEF combines evidence direction and stability into one reliability score:
\[
\operatorname{RES}(x)
=
\{1-C(x)\}\operatorname{Stab}(x).
\]
Equivalently, when $\varepsilon=0$,
\[
\operatorname{RES}(x)
=
|D(x)|\operatorname{Stab}(x).
\]

The score is high only when evidence points clearly in one direction and remains stable under perturbation. It is low when evidence is conflicted, unstable, or both.

\section{SEF-Audit and the Evidence Reliability Frontier}

The quantities above can be used for explanation, but they can also be used for action. We define SEF-Audit as a review triage rule. The goal is to decide which predictions are safe enough to accept automatically and which predictions should be reviewed.

\begin{definition}[SEF risk score]
The SEF risk score is
\[
R_{\mathrm{SEF}}(x)=1-\operatorname{RES}(x).
\]
Large values indicate high evidence conflict, low evidence stability, or both.
\end{definition}

\begin{definition}[SEF-Audit rule]
For a review budget $\beta\in[0,1]$, let $q_\beta$ be the $(1-\beta)$ quantile of $R_{\mathrm{SEF}}(X)$ in the audit population. SEF-Audit reviews a case when
\[
R_{\mathrm{SEF}}(x)>q_\beta
\]
and accepts it automatically otherwise.
\end{definition}

Thus, $\beta$ controls the fraction of cases sent to review. When $\beta=0$, no cases are reviewed. When $\beta$ grows, the rule reviews more cases with high evidence risk.

\begin{definition}[Evidence Reliability Frontier]
Let $L(x)$ be a loss indicator or bounded loss for the prediction at $x$. The Evidence Reliability Frontier is
\[
\mathcal{F}(\beta)
=
\frac{
\E\{L(X)\mathbf{1}[R_{\mathrm{SEF}}(X)\leq q_\beta]\}
}{
\Pp\{R_{\mathrm{SEF}}(X)\leq q_\beta\}
},
\]
whenever the denominator is positive.
\end{definition}

The frontier answers a practical question:

\begin{quote}
If we review the riskiest $\beta$ fraction of cases, what is the error rate among the cases we accept automatically?
\end{quote}

\begin{assumption}[Monotone evidence-risk condition]
Let
\[
\eta(r)=\E\{L(X)\mid R_{\mathrm{SEF}}(X)=r\}
\]
be the conditional prediction risk among cases with SEF risk value $r$. We say the audit population satisfies monotone evidence risk if $\eta(r)$ is nondecreasing in $r$.
\end{assumption}

This assumption is not automatic. It is exactly the condition the experiments are meant to check. When it holds, higher SEF risk means higher expected prediction loss.

\subsection{A finite-sample deployment diagnostic}

Let $(R_i,L_i)_{i=1}^m$ be a calibration sample that was not used to fit the prediction model, choose the evidence construction, or tune the review rule. Here $R_i$ is the proposed SEF review score and $L_i\in\{0,1\}$ is the observed error indicator. Let $U_i$ be the standardized rank of $R_i$, and define
\[
T_m=\frac{1}{m}\sum_{i=1}^m U_i\{L_i-\bar L\}.
\]
Positive values mean that errors tend to occur at larger review scores. We compute a one-sided permutation $p$-value by holding the scores fixed and permuting the losses.

\begin{proposition}[Finite-sample validity under the no-association null]
Suppose the calibration losses are exchangeable conditional on the observed review scores and are independent of their ordering under the null hypothesis. Then the permutation test that rejects for large $T_m$ has conditional type-I error at most $\alpha$, with randomized tie-breaking when exact level is required.
\end{proposition}

\begin{proof}
Conditional on the observed scores and the multiset of losses, every permutation of the losses is equally likely under the null. The rank of the observed statistic among its permutation values is therefore uniform up to ties. The usual permutation $p$-value is consequently super-uniform.
\end{proof}

This is a test of positive rank association, not a proof that the full conditional risk function $\eta(r)$ is monotone. It is a practical gate for the review rule. In our experiments we call the diagnostic adequately informed only when the calibration subset contains at least 20 errors, at least 20 non-errors, and at least 25 observations in each conflict quartile. These are transparent operating requirements, not universal power guarantees. If the sample is smaller, if the test does not reject, or if the estimated direction is negative, the high-conflict review rule is not activated.

\begin{theorem}[Optimal review under monotone evidence risk]
Assume monotone evidence risk and suppose $R_{\mathrm{SEF}}(X)$ has no point mass at the review threshold. Among all review rules that send a fraction $\beta$ of cases to review, SEF-Audit minimizes the expected loss among automatically accepted cases.
\end{theorem}

\begin{proof}
Let $A$ be the accepted set of any rule that accepts a fraction $1-\beta$ of cases. If $A$ contains a point with larger SEF risk than a point outside $A$, swapping them cannot increase the expected accepted loss because $\eta(r)$ is nondecreasing. Repeating this exchange moves all accepted cases below all reviewed cases in SEF risk. The resulting accepted set is exactly $\{x:R_{\mathrm{SEF}}(x)\leq q_\beta\}$, up to threshold ties. Therefore the SEF-Audit accepted set has the smallest expected accepted loss among rules with the same review budget.
\end{proof}

This theorem makes the role of SEF-Audit precise. The method is not claiming magic. It says that if SEF risk is aligned with error risk in a given domain, then the highest-SEF-risk cases are the right cases to review first. The repeated-split experiments below test that alignment rather than assuming it.

\begin{proposition}[Nested acceptance sets]
If $\beta_2\geq\beta_1$, then the SEF-Audit acceptance set at review budget $\beta_2$ is contained in the acceptance set at review budget $\beta_1$, up to quantile ties.
\end{proposition}

\begin{proof}
Increasing $\beta$ lowers the quantile threshold $q_\beta$ for accepted cases. Therefore, any case accepted under the larger review budget must also satisfy the weaker acceptance rule under the smaller review budget, except for possible ties at the quantile boundary.
\end{proof}

This nesting property is useful in deployment. It means analysts can increase the review budget without changing the ordering of cases. The riskiest cases are reviewed first.

\begin{proposition}[Positive screening improvement]
Suppose for a given review budget $\beta\in(0,1)$ that
\[
\E\{L(X)\mid R_{\mathrm{SEF}}(X)>q_\beta\}
>
\E\{L(X)\mid R_{\mathrm{SEF}}(X)\leq q_\beta\}.
\]
Then SEF-Audit at budget $\beta$ achieves strictly lower accepted-case loss than uniform random review at the same budget. This condition is weaker than monotone evidence risk and can be checked on a held-out calibration sample at a specific budget without requiring the full conditional risk function $\eta(r)$ to be nondecreasing.
\end{proposition}

\begin{proof}
Under uniform random review at budget $\beta$, the accepted fraction $1-\beta$ is drawn independently of $R_{\mathrm{SEF}}$, so the expected accepted-case loss equals $\E\{L(X)\}$. By the law of total expectation,
\[
\E\{L\}
=(1-\beta)\,\E\{L\mid R_{\mathrm{SEF}}\leq q_\beta\}
+\beta\,\E\{L\mid R_{\mathrm{SEF}}>q_\beta\}.
\]
The positive screening condition gives $\E\{L\mid R_{\mathrm{SEF}}>q_\beta\}>\E\{L\mid R_{\mathrm{SEF}}\leq q_\beta\}$, so the accepted-case mean must satisfy $\E\{L\mid R_{\mathrm{SEF}}\leq q_\beta\}<\E\{L\}$, which is the random-review baseline.
\end{proof}

This proposition is the minimal sufficient condition for SEF-Audit to improve over uninformed review. It does not require monotonicity, nor even a global rank association. The ScopeGate permutation diagnostic in the experiments below tests exactly this condition at a chosen budget.

\section{Conformal Reliability Screening}

The reliable evidence score is useful, but a raw score is not a guarantee. We add a conformal screen.

Let $\mathcal{I}_{cal}$ be a calibration set. For each calibration point, define an evidence unreliability score
\[
A_i
=
C(X_i)+\{1-\operatorname{Stab}(X_i)\}.
\]
Other bounded, label-free evidence-risk scores may be used. The score for a
future case must be computable without knowing that case's outcome.

Let $q_{1-\alpha}$ be the split-conformal quantile of $\{A_i:i\in\mathcal{I}_{cal}\}$:
\[
q_{1-\alpha}
=
\text{the }
\left\lceil (|\mathcal{I}_{cal}|+1)(1-\alpha)\right\rceil
\text{th order statistic.}
\]
If this rank equals $|\mathcal{I}_{cal}|+1$, we use the standard convention
$q_{1-\alpha}=+\infty$.

For a new point $x$, SEF flags the prediction as evidence-reliable when
\[
A(x)\leq q_{1-\alpha}.
\]

\begin{proposition}[Distribution-free screening guarantee]
Assume the calibration points and the new test point are exchangeable. Then
\[
\Pp\{A(X_{new})\leq q_{1-\alpha}\}\geq 1-\alpha.
\]
\end{proposition}

\begin{proof}
This is the standard split-conformal rank argument. Under exchangeability, the rank of the test score among the calibration scores and the test score is uniform. Choosing the conformal quantile gives the stated coverage guarantee.
\end{proof}

This does not say that every reliable flag is correct in a scientific sense. It says that the screening rule has a distribution-free calibration property for the chosen evidence-unreliability score.

The guarantee also inherits the exchangeability assumption stated in the proposition. This is reasonable for the static tabular settings studied here, where calibration and test cases come from the same pool. It fails when the data arrive as a stream and the underlying regime shifts, because the calibration pool can stop resembling the current data. In separate work we study spectrally weighted and drift-aware conformal procedures built for exactly that non-exchangeable case, where the relevance of past calibration residuals is reweighted using the spectral similarity of recent windows and the target miscoverage level is updated online as drift is detected \citep{opoku2026dasc,opoku2026spectral}. Combining that machinery with the evidence-unreliability score above would extend the screen to streaming settings. We do not pursue it here, and the guarantee in this paper should be read as conditional on exchangeability.

We also check the screen empirically (Table~\ref{tab:conformal-screen} and Figure~\ref{fig:conformal-screen}). On each standard benchmark, we split the data into training, calibration, and test parts. The training set fits the prediction model. The calibration set sets the conformal threshold for evidence unreliability at $\alpha=0.10$. The test set then shows which cases are accepted automatically and which cases are flagged.

\begin{table}[!htbp]
\centering
\caption{Empirical conformal-screen validation across 50 random splits. The conformal rule controls the fraction accepted by the evidence-unreliability threshold, while the error columns show whether flagged cases are actually riskier.}
\label{tab:conformal-screen}
\resizebox{\textwidth}{!}{\begin{tabular}{lrrrr}
\toprule
Dataset & Accepted fraction & All error & Accepted error & Flagged error \\
\midrule
Breast cancer & 0.913 & 0.029 & 0.011 & 0.226 \\
Digits & 0.915 & 0.008 & 0.000 & 0.104 \\
Iris & 0.974 & 0.053 & 0.042 & 0.409 \\
Wine & 0.948 & 0.030 & 0.030 & 0.011 \\
\bottomrule
\end{tabular}
}
\end{table}

\begin{figure}[!htbp]
\centering
\includegraphics[width=0.86\textwidth]{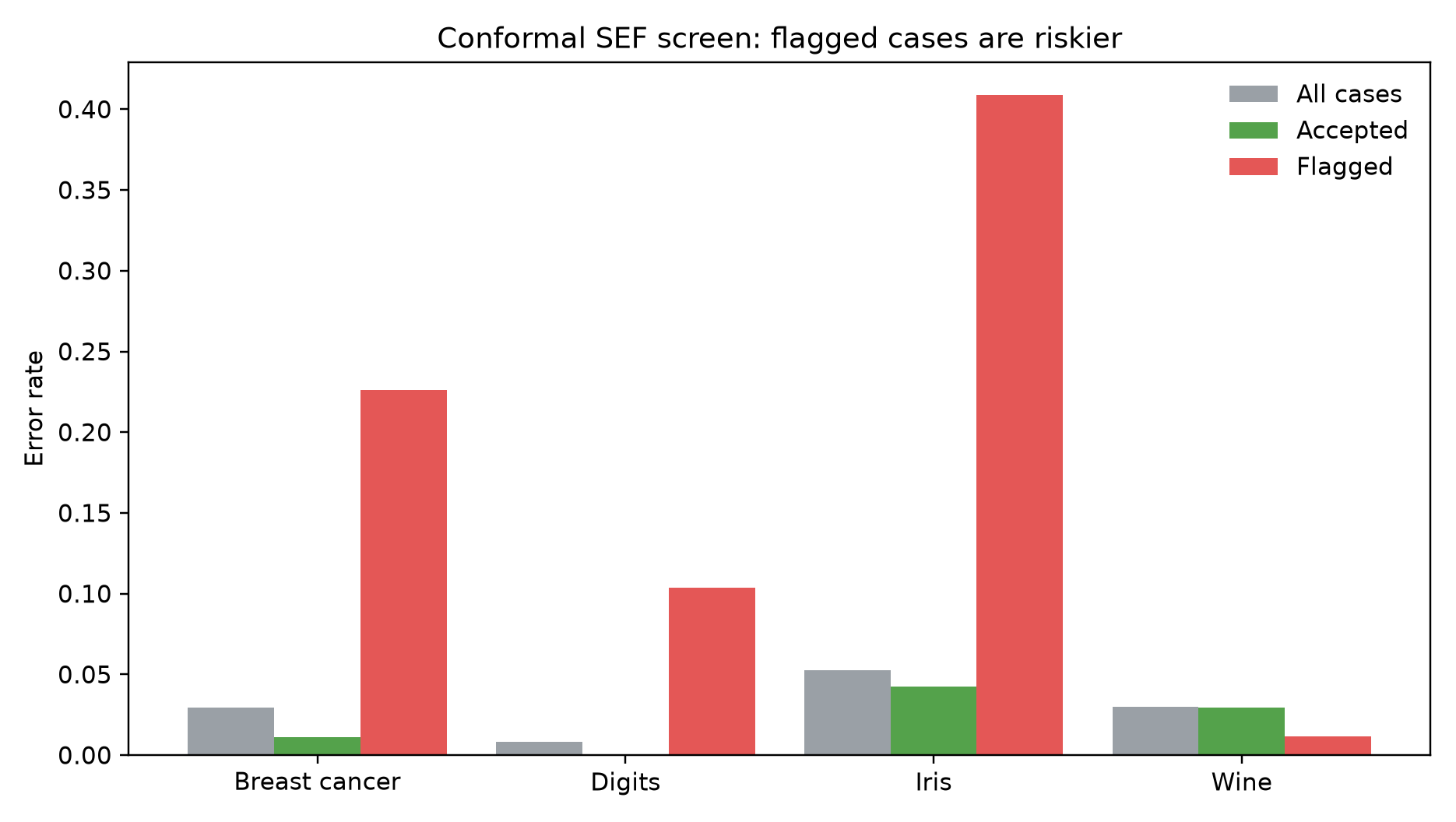}
\caption{Empirical conformal-screen validation. Flagged cases have much higher error than accepted cases on Breast Cancer, Digits, and Iris. Wine is the exception in this small benchmark, which is why the screen should be calibrated and checked for each application.}
\label{fig:conformal-screen}
\end{figure}

The result is not perfect, and it should not be oversold. The screen works clearly on three of the four standard tasks. On Wine, the flagged group is small and not more error-prone on average. A likely explanation is the low residual error in this binary Wine task: after the model separates most observations correctly, too few errors remain for the flagged tail to show a stable enrichment pattern. This is still useful because it shows that the conformal screen gives a calibrated evidence flag, not a guarantee that every flagged group will have higher prediction error in every small data set.

\section{Algorithm}

\begin{enumerate}
\item Fit any prediction model $f$.
\item Choose a reference distribution $P_0$.
\item For each point $x$, compute signed evidence terms $E_j(x)$.
\item Split evidence into $E_j^+(x)$ and $E_j^-(x)$.
\item Compute $S^+(x)$, $S^-(x)$, $M(x)$, $D(x)$, and $C(x)$.
\item Generate $B$ perturbations and recompute evidence terms.
\item Estimate $\operatorname{Stab}_j(x)$ and $\operatorname{Stab}(x)$.
\item Compute $\operatorname{RES}(x)$.
\item Use calibration data to set a conformal evidence-reliability threshold.
\item For SEF-Audit, rank cases by $R_{\mathrm{SEF}}(x)=1-\operatorname{RES}(x)$ and review the highest-risk cases under the chosen review budget.
\end{enumerate}

\section{Synthetic Demonstration}

We ran a first synthetic experiment to test the core claim that model confidence and evidence reliability are not the same thing. The data contain two leading signals. One signal supports the positive class and the other opposes it. A logistic score is then formed as
\[
\eta(x)=2.2x_1-2.0x_2+0.25x_3,
\]
with the class label drawn from a Bernoulli model with success probability $\{1+\exp[-\eta(x)]\}^{-1}$. This creates cases where the model can be confident even though large pieces of evidence push in opposite directions.

For this demonstration, the signed evidence terms are exact because the logit is linear:
\[
E_j(x)=\beta_j(x_j-\bar{x}_j).
\]
We then compute support, opposition, conflict, perturbation stability, and the reliable evidence score. Stability is estimated by perturbing the coefficient vector and reference baseline 150 times.

\begin{figure}[!htbp]
\centering
\includegraphics[width=0.82\textwidth]{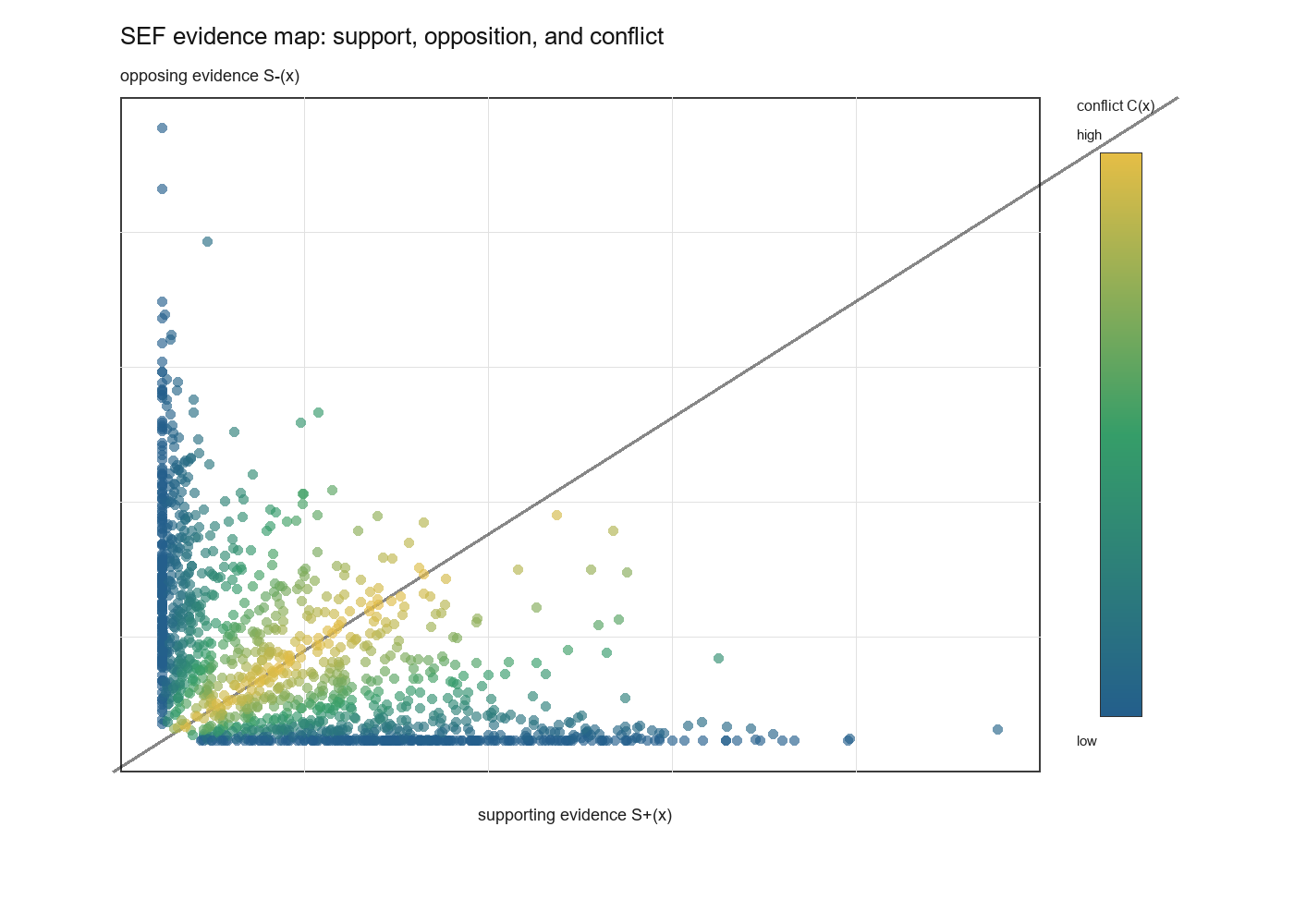}
\caption{SEF evidence map. Each point is one case. The horizontal axis shows supporting evidence and the vertical axis shows opposing evidence. Points near the diagonal have high conflict because support and opposition are both strong.}
\label{fig:evidence-map}
\end{figure}

\begin{figure}[!htbp]
\centering
\includegraphics[width=0.82\textwidth]{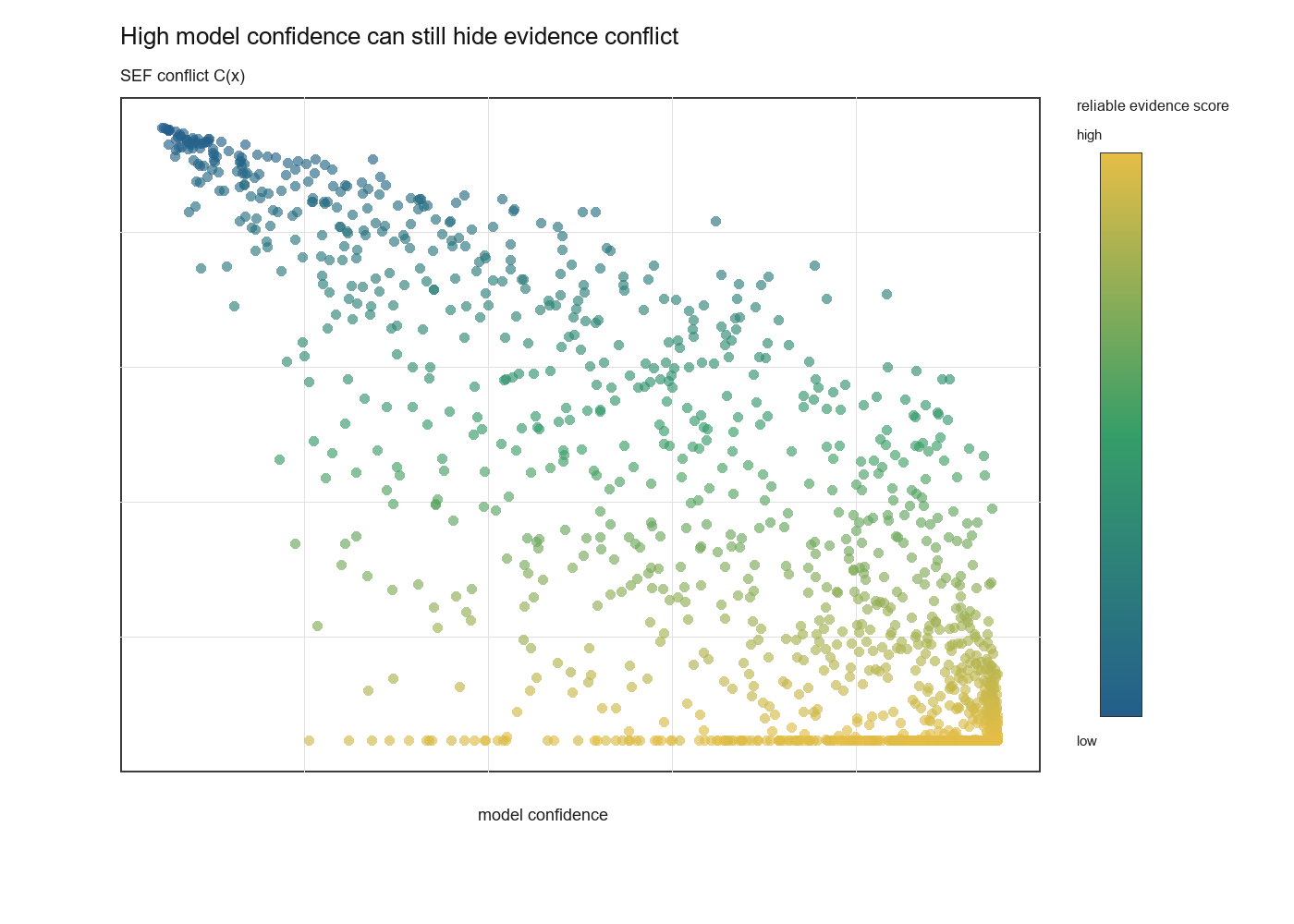}
\caption{Model confidence does not remove evidence conflict. Some cases have high prediction confidence and high SEF conflict at the same time. These are cases where the model score is decisive, but the evidence is internally mixed.}
\label{fig:confidence-conflict}
\end{figure}

\begin{figure}[!htbp]
\centering
\includegraphics[width=0.82\textwidth]{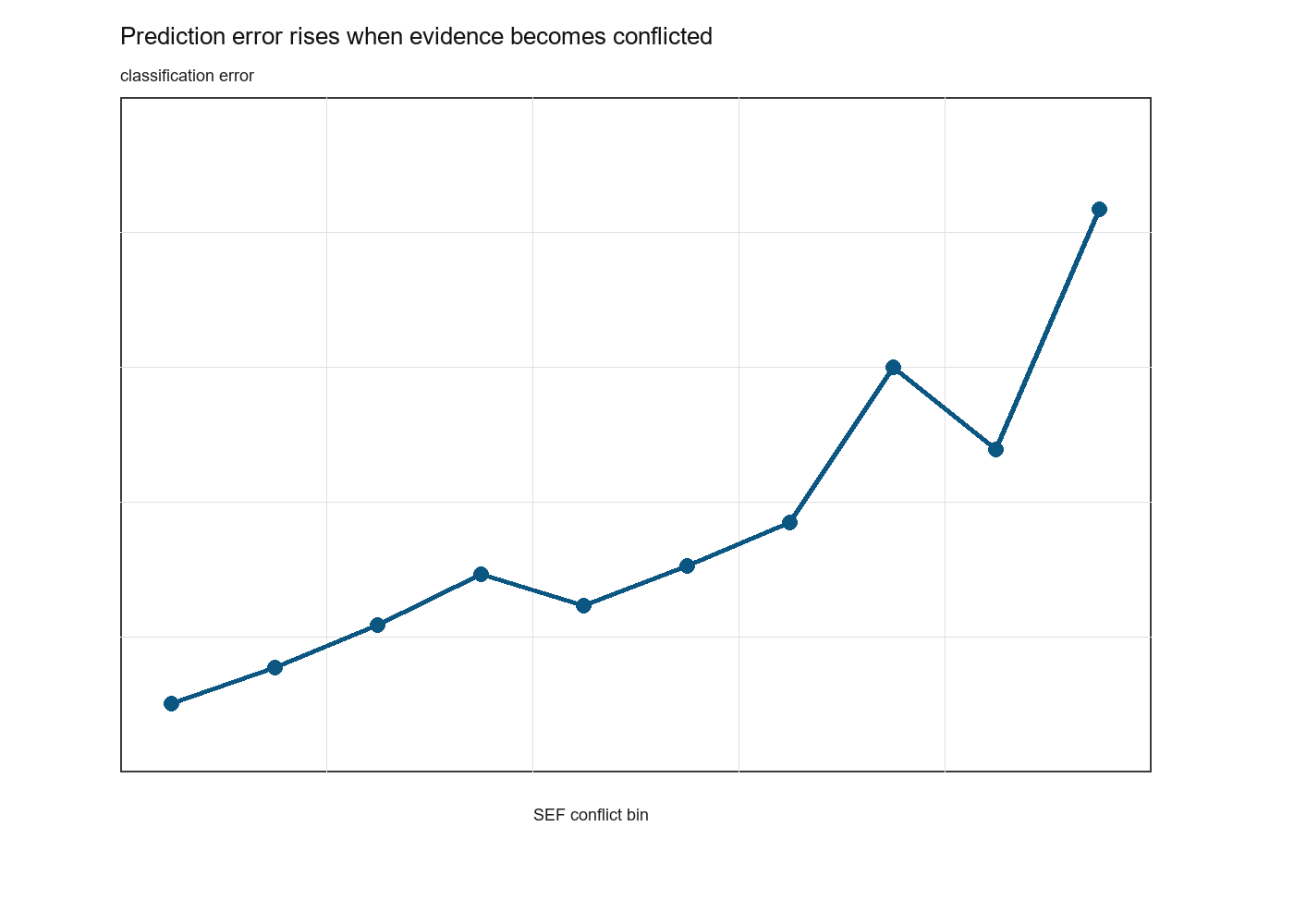}
\caption{Prediction error rises as SEF conflict increases. In this run, the overall error rate is $0.162$, the low-conflict error rate is $0.062$, and the high-conflict error rate is $0.358$.}
\label{fig:conflict-error}
\end{figure}

\begin{figure}[!htbp]
\centering
\includegraphics[width=0.82\textwidth]{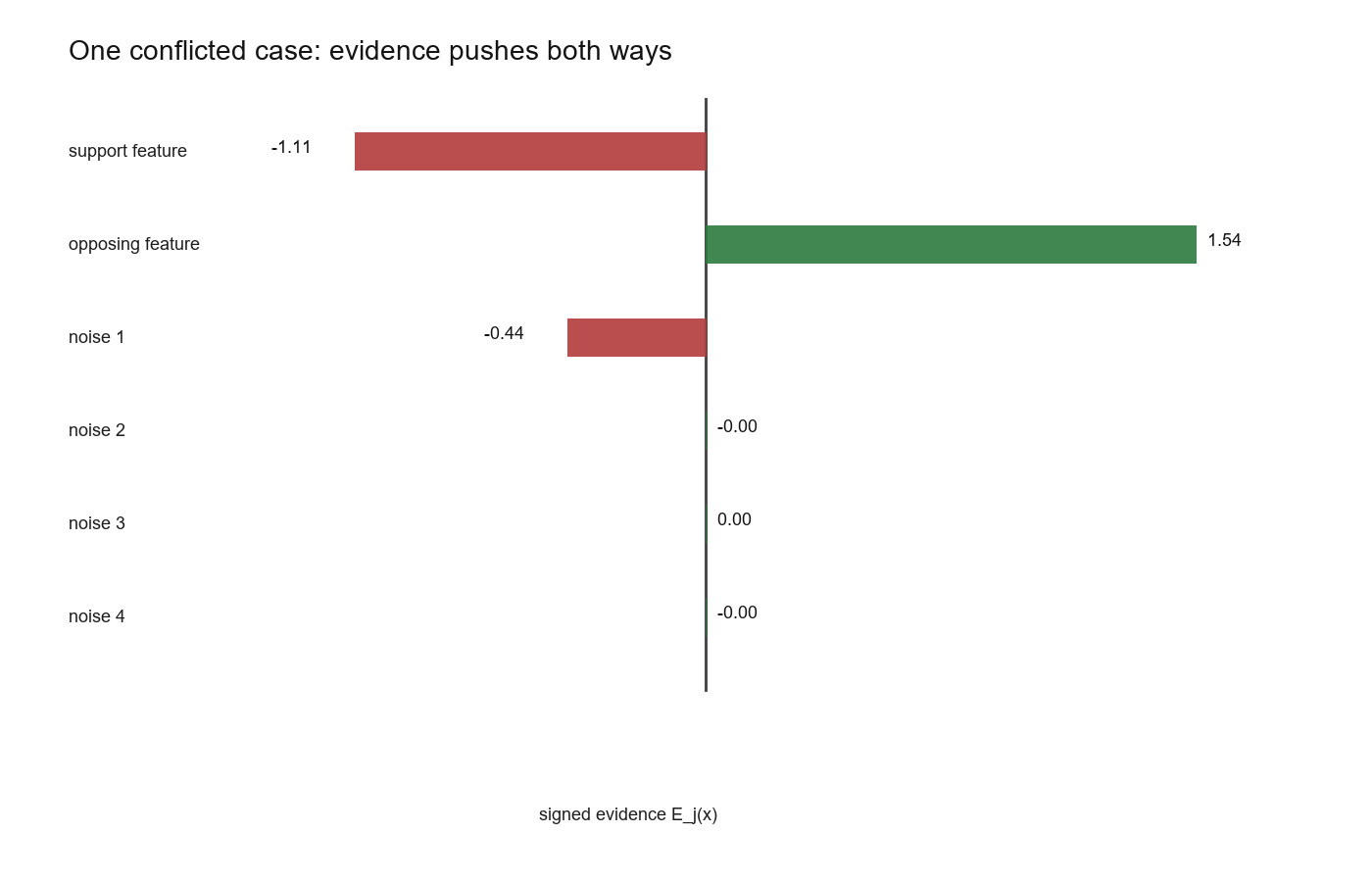}
\caption{A single high-conflict case. The prediction is formed from signed evidence terms that push in both directions. SEF makes this internal opposition visible.}
\label{fig:single-case}
\end{figure}

Figures~\ref{fig:evidence-map}, \ref{fig:confidence-conflict}, \ref{fig:conflict-error}, and~\ref{fig:single-case} illustrate the synthetic results. The overall classification error is $0.162$. Among the lowest-conflict quartile, the error is $0.062$. Among the highest-conflict quartile, the error rises to $0.358$. The reliable evidence score is also informative: the highest-RES quartile has error $0.043$, while the lowest-RES quartile has error $0.360$.

These results do not yet prove that SEF improves every prediction system. They show something narrower and important: evidence conflict carries information that ordinary model confidence can hide.

\FloatBarrier

\section{Sanity Checks on Standard Benchmark Data Sets}

The synthetic example is useful because the true evidence structure is known. We next use four familiar scikit-learn tasks as implementation sanity checks: Iris versicolor versus virginica, Wine class 1 versus class 2, Breast Cancer Wisconsin benign versus malignant, and handwritten Digits 3 versus 5. These low-noise tasks are not the main evidence for practical value. Their near-zero error rates create ceiling effects, so the healthcare, Covertype, black-box, and ScopeGate studies carry more weight.

For each data set, we fit a standardized logistic regression model over 50 stratified train-test splits. We then compute signed evidence terms from the fitted logit,
\[
E_j(x)=\widehat{\beta}_j(x_j-\bar{x}_j),
\]
where variables are standardized using the training data. Stability is estimated by bootstrap refitting the full preprocessing and logistic model. The purpose of the benchmark is not to claim that logistic regression is the best possible predictor. The purpose is to test whether SEF conflict and SEF reliability identify riskier predictions on familiar data.

\begin{table}[!htbp]
\centering
\caption{Repeated standard benchmark results over 50 stratified train-test splits. Values are mean $\pm$ approximate 95\% confidence intervals across splits. Low-conflict and high-conflict refer to the bottom and top quartiles of SEF conflict. High-RES and low-RES refer to the top and bottom quartiles of the reliable evidence score. Success is the fraction of splits where high-conflict cases have higher error than low-conflict cases.}
\label{tab:standard-benchmarks}
\resizebox{\textwidth}{!}{\begin{tabular}{lrrrrrr}
\toprule
Dataset & Error & Low conflict & High conflict & High RES & Low RES & Success \\
\midrule
Iris: versicolor vs virginica & 0.049 $\pm$ 0.009 & 0.000 $\pm$ 0.000 & 0.184 $\pm$ 0.032 & 0.000 $\pm$ 0.000 & 0.189 $\pm$ 0.034 & 0.90 \\
Wine: class 1 vs class 2 & 0.029 $\pm$ 0.006 & 0.000 $\pm$ 0.000 & 0.078 $\pm$ 0.022 & 0.000 $\pm$ 0.000 & 0.082 $\pm$ 0.022 & 0.58 \\
Breast cancer: benign vs malignant & 0.025 $\pm$ 0.003 & 0.000 $\pm$ 0.000 & 0.094 $\pm$ 0.010 & 0.000 $\pm$ 0.000 & 0.094 $\pm$ 0.011 & 1.00 \\
Digits: 3 vs 5 & 0.006 $\pm$ 0.002 & 0.000 $\pm$ 0.000 & 0.024 $\pm$ 0.008 & 0.000 $\pm$ 0.000 & 0.024 $\pm$ 0.008 & 0.50 \\
\bottomrule
\end{tabular}
}
\end{table}

\begin{figure}[!htbp]
\centering
\includegraphics[width=0.86\textwidth]{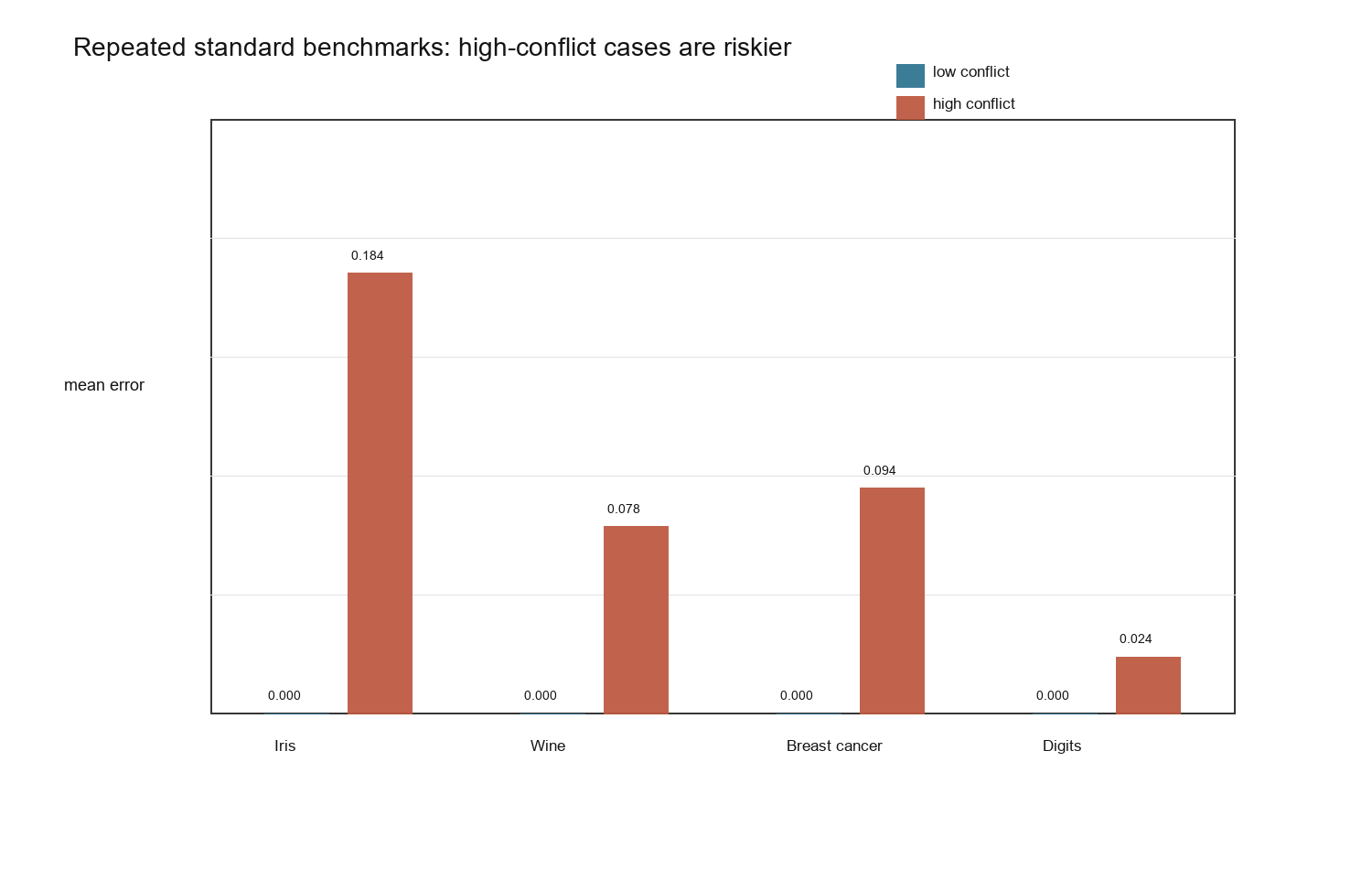}
\caption{Implementation sanity checks on clean standard data sets. The near-zero low-conflict errors are ceiling effects and should not be read as strong evidence of general risk separation.}
\label{fig:standard-conflict}
\end{figure}

\begin{figure}[!htbp]
\centering
\includegraphics[width=0.86\textwidth]{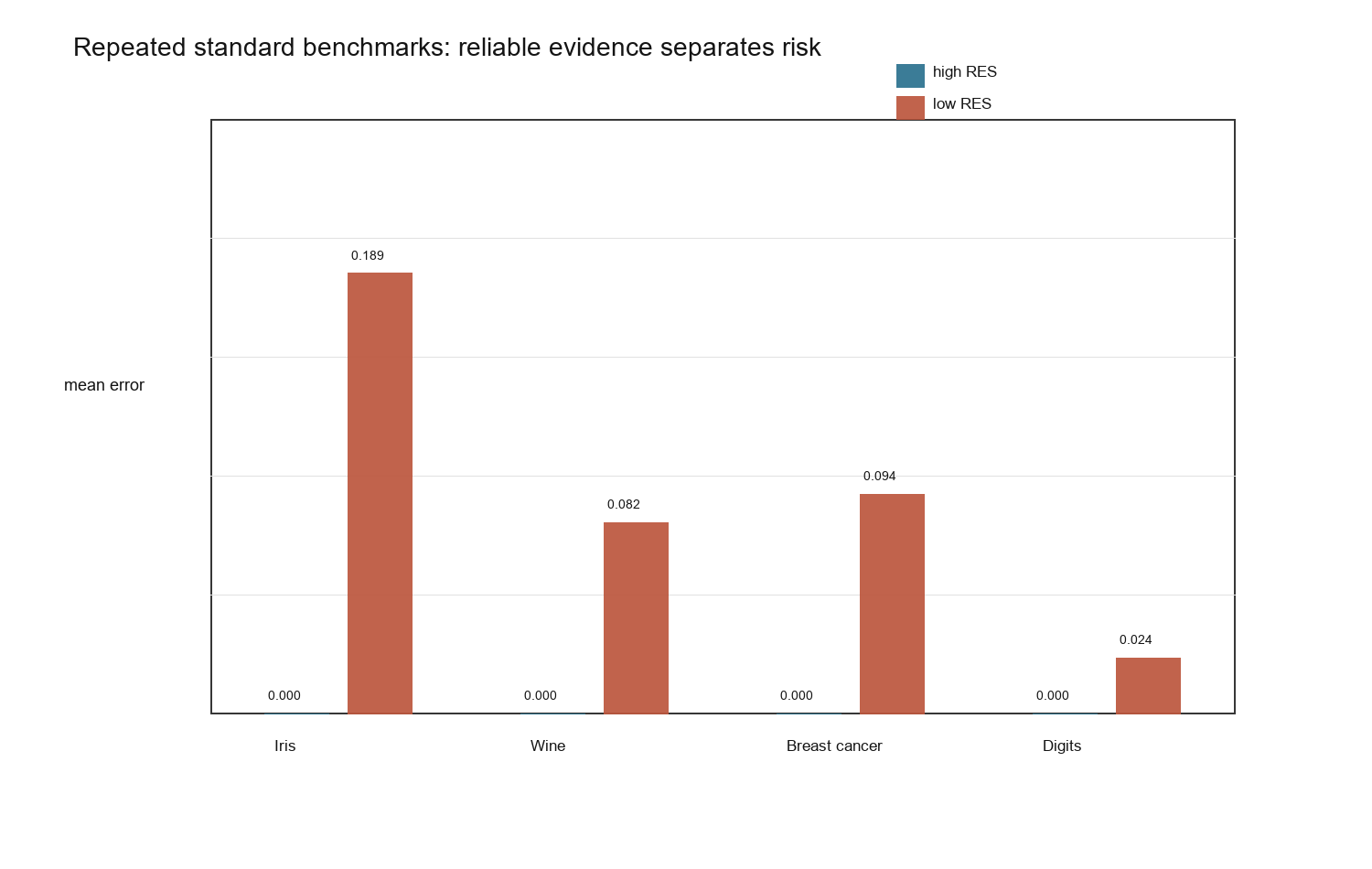}
\caption{Repeated standard benchmark data sets. The reliable evidence score separates safer cases from less reliable cases on average across 50 splits.}
\label{fig:standard-res}
\end{figure}

\begin{figure}[!htbp]
\centering
\includegraphics[width=0.86\textwidth]{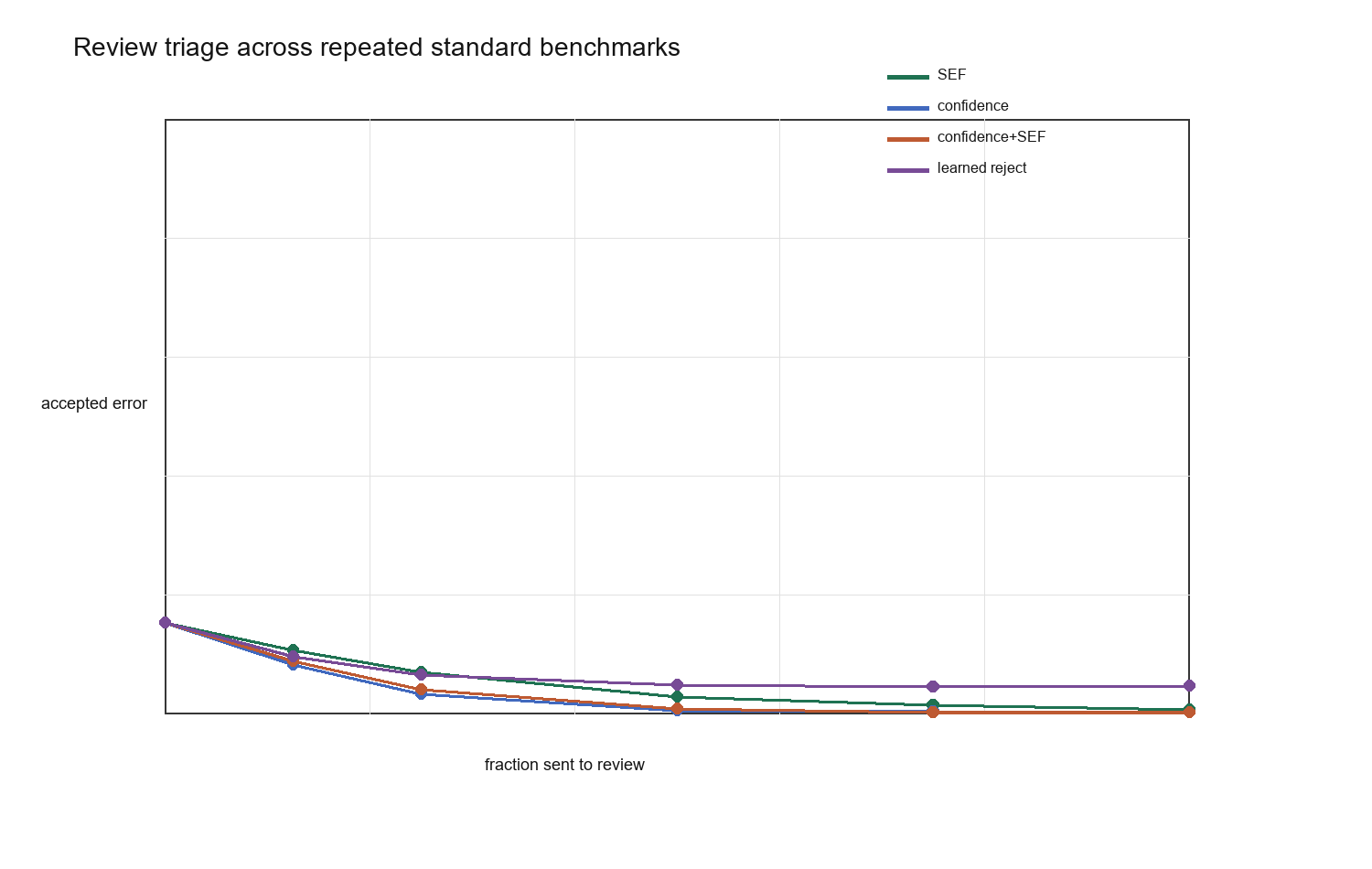}
\caption{Review triage across the repeated standard benchmark tasks. Cases are ranked by SEF risk, confidence risk, a combined confidence-plus-SEF score, or a learned reject baseline.}
\label{fig:triage-curve}
\end{figure}

The benchmark results are encouraging (Table~\ref{tab:standard-benchmarks}, Figures~\ref{fig:standard-conflict} and~\ref{fig:standard-res}), and they are now less dependent on chance splits. On Iris, the overall error is $0.049$, while the high-conflict quartile has error $0.184$. On Wine, the overall error is $0.029$, while the high-conflict quartile has error $0.078$. On Breast Cancer Wisconsin, the overall error is $0.025$, while the high-conflict quartile has error $0.094$. On Digits 3 versus 5, the overall error is $0.006$, while the high-conflict quartile has error $0.024$. Across these tasks, the low-conflict and high-RES quartiles have near-zero average error.

The reliable evidence score gives a similar pattern. These are clean benchmark tasks with low base error, so their role is to confirm the mechanics of the method before the healthcare and large-data benchmarks.

Figure~\ref{fig:triage-curve} shows the operational view. If the analyst sends the highest-risk cases to review and accepts the rest automatically, accepted-case error drops as the review budget grows. On these clean data sets, confidence is a very strong triage rule; SEF adds an evidence audit layer rather than replacing confidence-based selective prediction.

\begin{table}[!htbp]
\centering
\caption{Selective-review comparison at a fixed 30\% review budget across 50 splits. Lower values are better. The learned reject baseline is trained on a held-out part of the training data to predict whether the base classifier will make a mistake.}
\label{tab:selective-baseline}
\resizebox{\textwidth}{!}{\begin{tabular}{lrrrr}
\toprule
Dataset & SEF & Confidence & Confidence+SEF & Learned reject \\
\midrule
Iris: versicolor vs virginica & 0.000 $\pm$ 0.000 & 0.000 $\pm$ 0.000 & 0.000 $\pm$ 0.000 & 0.018 $\pm$ 0.008 \\
Wine: class 1 vs class 2 & 0.008 $\pm$ 0.004 & 0.000 $\pm$ 0.000 & 0.000 $\pm$ 0.000 & 0.003 $\pm$ 0.003 \\
Breast cancer: benign vs malignant & 0.002 $\pm$ 0.001 & 0.003 $\pm$ 0.001 & 0.002 $\pm$ 0.001 & 0.010 $\pm$ 0.003 \\
Digits: 3 vs 5 & 0.000 $\pm$ 0.000 & 0.000 $\pm$ 0.000 & 0.000 $\pm$ 0.000 & 0.002 $\pm$ 0.001 \\
\bottomrule
\end{tabular}
}
\end{table}

Table~\ref{tab:selective-baseline} answers a reviewer concern directly. Confidence and confidence-plus-SEF are hard to beat on the clean standard tasks. SEF is competitive on Breast Cancer and useful as an interpretable review score, but it is not a universal substitute for confidence-based selective prediction. The learned reject baseline uses a held-out $35\%$ portion of the training set. A gradient-boosted reject model with 40 depth-two trees predicts whether the base classifier will err from its confidence and absolute probability margin. This architecture and every random seed are fixed in the experiment script. The later confidence-masked and healthcare experiments ask a sharper question: after confidence has already selected cases that look safe, can SEF still find hidden risk?

\FloatBarrier

\section{Large Real-Data Benchmark}

To test whether SEF can handle a larger real data set, we also use the Covertype data set from the UCI Machine Learning Repository \citep{blackard1998covtype}. The full data set has $581{,}012$ observations and $54$ features. We use the task of predicting whether a forest cover observation belongs to cover type 2 versus all other cover types. For computational speed, the experiment draws a stratified working sample of $120{,}000$ observations and evaluates performance on a held-out test set.

The base predictor is a standardized logistic model fit with stochastic gradient descent. Signed evidence is computed from the fitted logit coefficients,
\[
E_j(x)=\widehat{\beta}_j(x_j-\bar{x}_j),
\]
and SEF stability is estimated by refitting the preprocessing and model on bootstrap subsamples.

\begin{table}[!htbp]
\centering
\caption{Covertype benchmark with 95\% nonparametric bootstrap confidence intervals computed on the held-out test predictions.}
\label{tab:covtype-ci}
\begin{tabular}{lrr}
\toprule
Quantity & Estimate & 95\% bootstrap CI \\
\midrule
Overall error & 0.249 & [0.245, 0.253] \\
Low conflict error & 0.096 & [0.090, 0.101] \\
High conflict error & 0.409 & [0.399, 0.418] \\
High RES error & 0.094 & [0.089, 0.100] \\
Low RES error & 0.409 & [0.401, 0.418] \\
\bottomrule
\end{tabular}

\end{table}

Table~\ref{tab:covtype-ci} reports the large-data result, which is consistent with the smaller benchmarks. The overall test error is $0.249$ with a 95\% bootstrap interval of $[0.245,0.253]$. The low-conflict quartile has error $0.096$ with interval $[0.090,0.101]$, while the high-conflict quartile has error $0.409$ with interval $[0.399,0.418]$. The highest-RES quartile has error $0.094$, while the lowest-RES quartile has error $0.409$. The intervals are narrow and well separated. Thus, on this large real data set, SEF separates cases with clean, stable evidence from cases with more fragile evidence.

\begin{figure}[!htbp]
\centering
\includegraphics[width=0.86\textwidth]{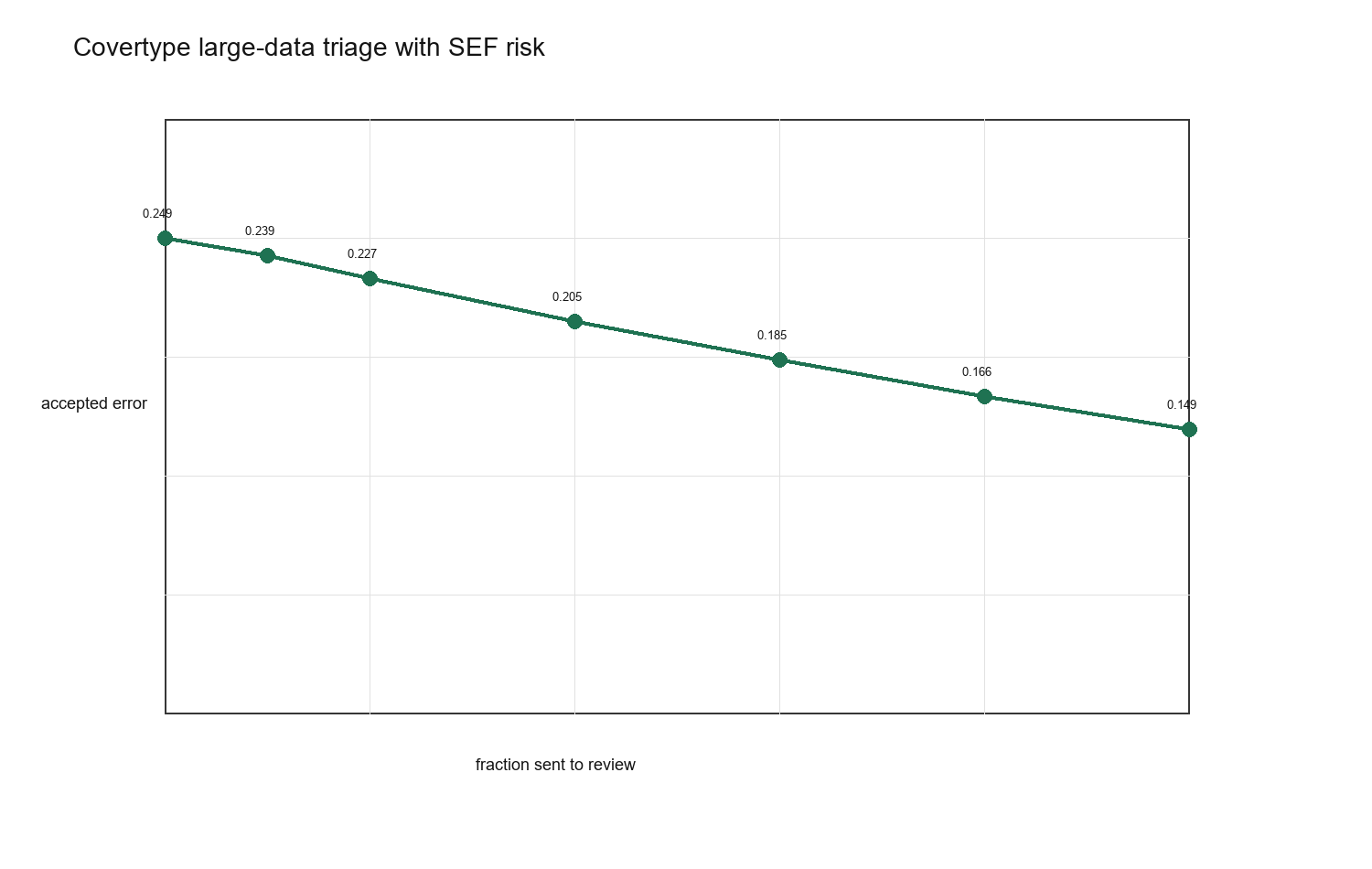}
\caption{Large Covertype benchmark. The full data source has $581{,}012$ rows; this experiment uses a stratified working sample of $120{,}000$ rows. Sending the highest SEF-risk cases to review steadily lowers accepted-case error.}
\label{fig:covtype-triage}
\end{figure}

The triage curve in Figure~\ref{fig:covtype-triage} gives the operational view. Without review, the accepted-case error is $0.249$ with interval $[0.245,0.253]$. Reviewing the riskiest $20\%$ by SEF risk lowers accepted-case error to $0.205$ with interval $[0.201,0.209]$. Reviewing the riskiest $50\%$ lowers accepted-case error to $0.149$ with interval $[0.144,0.153]$. This result matters because it shows that SEF can be useful on data sets that are much larger than the toy examples often used for explanation methods.

\FloatBarrier

\section{Multi-Class Empirical Validation}

The multi-class theorem should be checked on real multi-class predictions, not left as a formal extension. We therefore evaluate pairwise SEF using the predicted class and its strongest rival. The evidence terms come from the corresponding pairwise score gap of a multinomial logistic model. We use full three-class Iris, full three-class Wine, ten-class Digits, and seven-class Covertype. Iris, Wine, and Digits use 30 stratified splits. The larger Covertype experiment uses five stratified splits of a $97{,}950$-row working sample.

\begin{table}[!htbp]
\centering
\caption{Multi-class SEF validation. HC denotes the most confident half of predictions. Success is the fraction of splits where high-conflict confident cases have higher error than low-conflict confident cases. Hybrid review combines confidence risk and SEF conflict at a 30\% review budget.}
\label{tab:multiclass}
\resizebox{\textwidth}{!}{\begin{tabular}{lrrrrrr}
\toprule
Dataset & Classes & Error & HC low conflict & HC high conflict & Success & Hybrid review \\
\midrule
Covertype (7 classes) & 7 & 0.324 & 0.065 & 0.307 & 1.00 & 0.241 \\
Digits (10 classes) & 10 & 0.033 & 0.000 & 0.002 & 0.13 & 0.001 \\
Iris (3 classes) & 3 & 0.045 & 0.000 & 0.000 & 0.00 & 0.000 \\
Wine (3 classes) & 3 & 0.024 & 0.000 & 0.000 & 0.00 & 0.000 \\
\bottomrule
\end{tabular}
}
\end{table}

\begin{figure}[!htbp]
\centering
\includegraphics[width=0.90\textwidth]{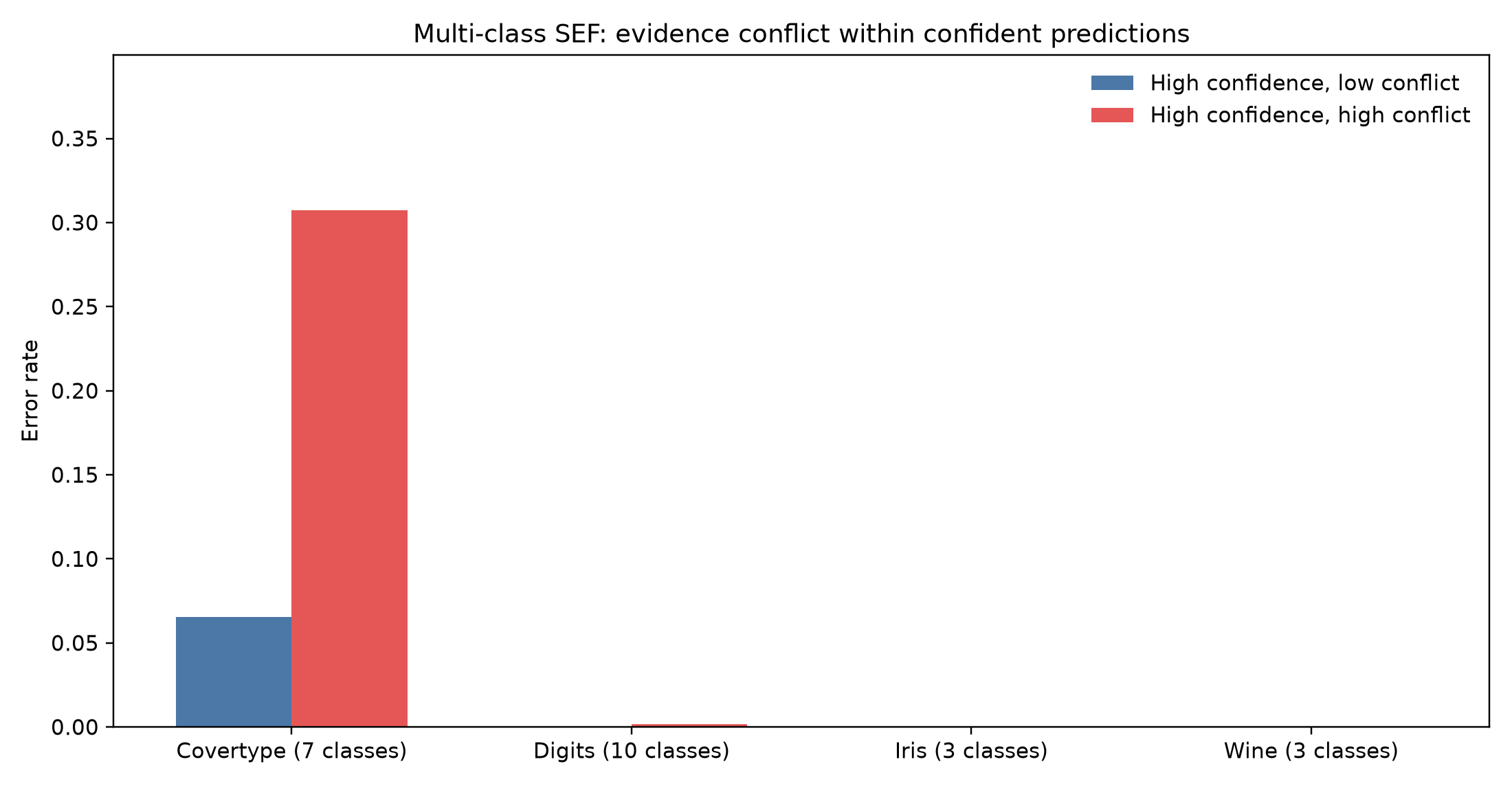}
\caption{Multi-class evidence conflict inside predictions that already look confident. Covertype gives the clearest nontrivial test because the confident subset still contains enough errors to compare.}
\label{fig:multiclass}
\end{figure}

Table~\ref{tab:multiclass} and Figure~\ref{fig:multiclass} present the multi-class results. The seven-class Covertype result supports the multi-class theory in a difficult real task. Within the most confident half of predictions, the low-conflict quartile has error $0.065$, while the high-conflict quartile has error $0.307$. The high-conflict group is riskier in all five splits. At a 30\% review budget, SEF-only accepted error is $0.247$, confidence-only error is $0.241$, and the combined rule gives $0.241$. Thus, conflict gives strong diagnostic separation, while confidence remains slightly better as a stand-alone review score.

The smaller data sets show a ceiling effect. The confident halves of Iris and Wine contain no errors, so neither conflict nor confidence can separate risk there. Digits contains very few errors in its confident half, and split success is correspondingly unstable. These results are still useful because they show why a multi-class audit needs a task with enough residual error to be informative.

\FloatBarrier

\section{Model-Agnostic Black-Box Robustness}

The previous benchmark sections use either known linear evidence terms or logistic coefficients. To test the model-agnostic form of SEF, we also run a repeated robustness experiment with non-linear black-box predictors. The experiment uses the four standard benchmark tasks in Table~\ref{tab:standard-benchmarks}, two fitted models, and 20 random train-test splits. The two models are random forests and histogram gradient boosting. Across four data sets, two models, and 20 seeds, this gives 160 fitted prediction problems.

For this experiment, the SEF evidence term is computed by median-reference feature replacement:
\[
E_j(x)=f(x)-f(x_{-j},\widetilde{x}_j),
\]
where $\widetilde{x}_j$ is the training median of feature $j$ and $f(x)$ is the fitted model logit. This construction does not use model coefficients and can be applied to any classifier with predicted probabilities.

\begin{table}[!htbp]
\centering
\caption{Model-agnostic SEF robustness across 20 seeds. Values are unweighted means over repeated train-test splits within each data set and model class. The SEF evidence terms are computed by feature replacement, not by model coefficients.}
\label{tab:model-agnostic}
\resizebox{\textwidth}{!}{\begin{tabular}{llrrrr}
\toprule
Dataset & Model & Error & Low conflict & High conflict & High RES \\
\midrule
Breast cancer & Gradient boosting & 0.035 & 0.000 & 0.101 & 0.000 \\
Breast cancer & Random forest & 0.045 & 0.002 & 0.111 & 0.002 \\
Digits & Gradient boosting & 0.015 & 0.000 & 0.061 & 0.000 \\
Digits & Random forest & 0.010 & 0.000 & 0.037 & 0.000 \\
Iris & Gradient boosting & 0.069 & 0.015 & 0.155 & 0.013 \\
Iris & Random forest & 0.063 & 0.018 & 0.150 & 0.017 \\
Wine & Gradient boosting & 0.036 & 0.000 & 0.112 & 0.000 \\
Wine & Random forest & 0.025 & 0.000 & 0.059 & 0.000 \\
\bottomrule
\end{tabular}
}
\end{table}

\begin{figure}[!htbp]
\centering
\includegraphics[width=0.86\textwidth]{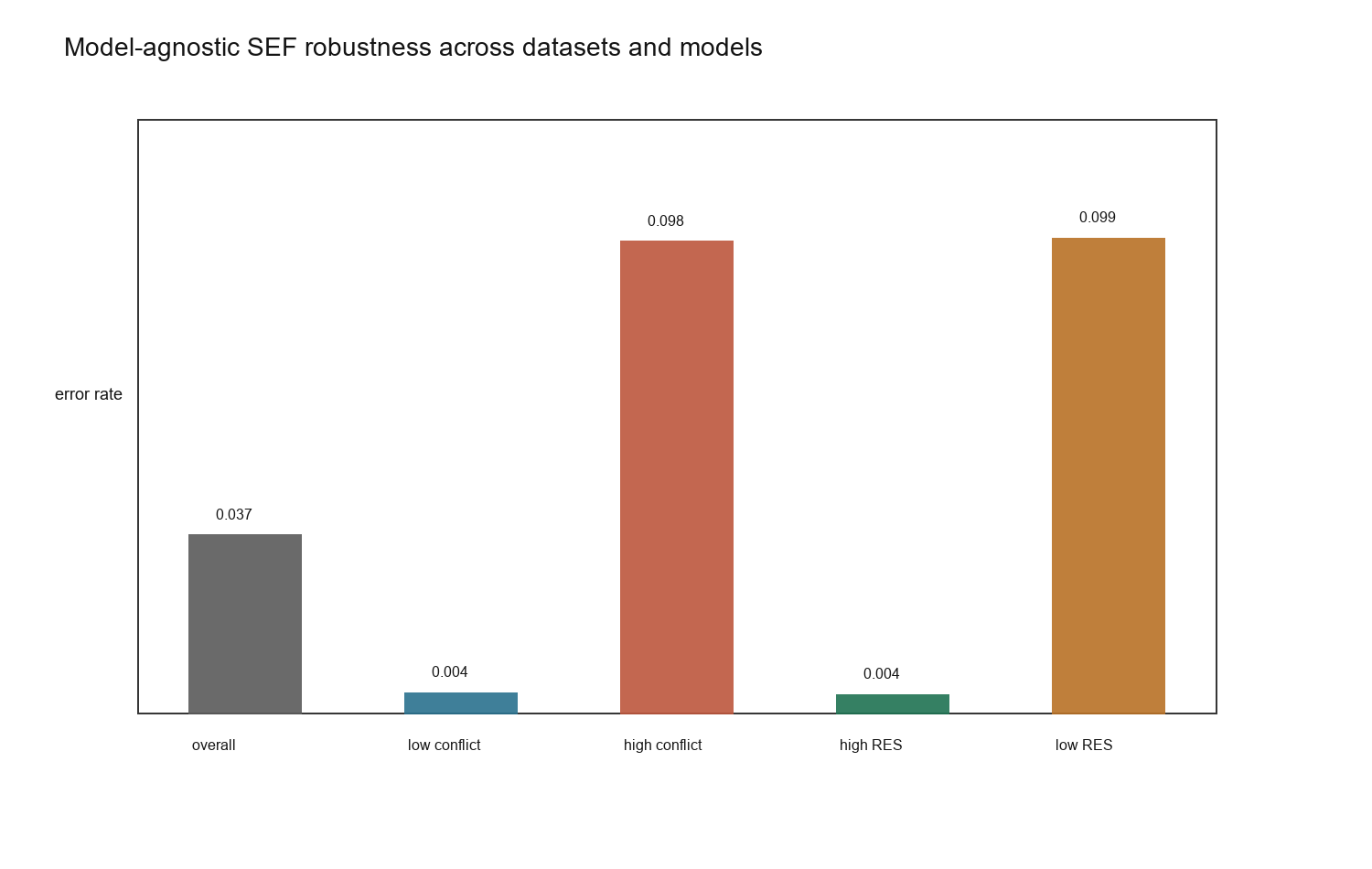}
\caption{Model-agnostic robustness across four standard data sets, two non-linear black-box model classes, and 20 random splits. High-conflict predictions have higher error than low-conflict predictions, even when SEF uses only feature-replacement contrasts.}
\label{fig:model-agnostic}
\end{figure}

Figure~\ref{fig:model-agnostic} visualizes the separation. The unweighted average over the eight data-set-by-model rows in Table~\ref{tab:model-agnostic} gives overall error $0.037$, low-conflict error $0.004$, and high-conflict error $0.098$. Thus, high-conflict predictions are about $24$ times as error-prone as low-conflict predictions in this repeated black-box experiment. The absolute ratio is large because several low-conflict cells have near-zero error, so the safer reading is the direction and separation, not the exact multiplier. This section is important because it shows that SEF is not only a linear-model diagnostic. The conflict idea remains useful when evidence is computed through model-agnostic perturbations.

On these clean benchmark tasks, ordinary model confidence remains a strong triage baseline. SEF adds value by explaining why a prediction may be fragile: the evidence itself is internally conflicted.

\FloatBarrier

\section{Ablation Study}

We next ask which part of SEF drives the review-triage gains. On the standard benchmark tasks, we compare six review rules at a fixed $30\%$ review budget: no review, conflict only, stability only, the full reliable evidence score, confidence only, and an equal-weighted confidence-plus-RES score. We repeat all four tasks over 50 train-test splits, giving 200 fitted prediction problems per review rule.

\begin{table}[!htbp]
\centering
\caption{Ablation study at a $30\%$ review budget over 200 fitted prediction problems. Values are mean accepted-case error $\pm$ approximate 95\% confidence intervals.}
\label{tab:ablation}
\begin{tabular}{lr}
\toprule
Risk score & Accepted error \\
\midrule
no review & 0.027 $\pm$ 0.003 \\
conflict only & 0.003 $\pm$ 0.001 \\
stability only & 0.011 $\pm$ 0.002 \\
RES & 0.003 $\pm$ 0.001 \\
confidence & 0.001 $\pm$ 0.000 \\
confidence + RES & 0.001 $\pm$ 0.000 \\
\bottomrule
\end{tabular}

\end{table}

\begin{figure}[!htbp]
\centering
\includegraphics[width=0.86\textwidth]{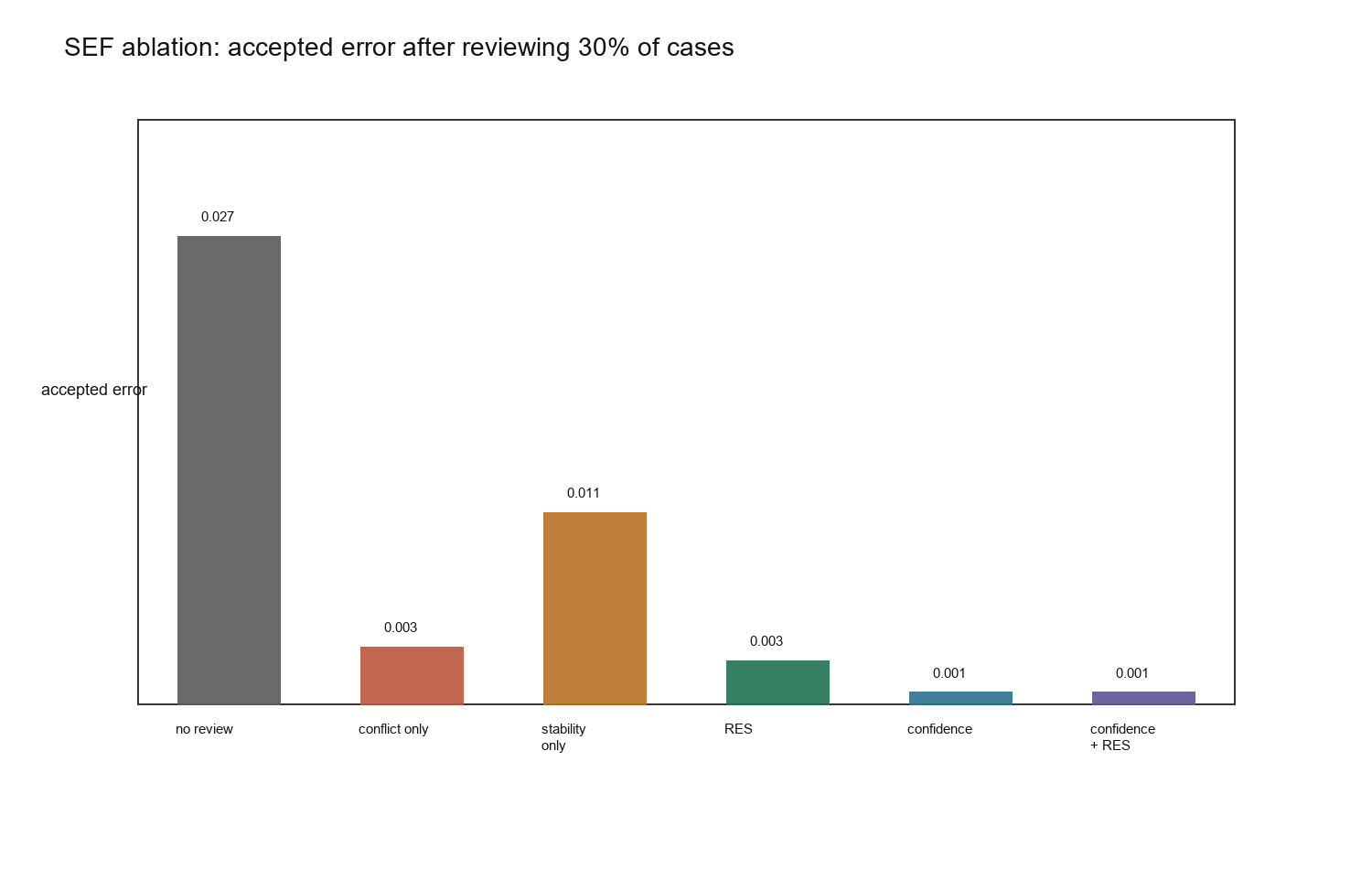}
\caption{Repeated SEF ablation study. Conflict alone and the full reliable evidence score reduce accepted-case error relative to no review. Confidence remains very strong on these clean benchmark tasks.}
\label{fig:ablation}
\end{figure}

Table~\ref{tab:ablation} and Figure~\ref{fig:ablation} report the repeated ablation, which gives two lessons. First, conflict carries useful information: using conflict alone lowers accepted-case error from $0.027$ to $0.003$. The full RES score also gives error $0.003$, while stability alone gives $0.011$. Second, ordinary model confidence remains the strongest single baseline on these clean benchmark tasks, with accepted-case error near $0.001$. SEF adds a complementary diagnostic layer that explains whether the evidence behind a prediction is one-sided or internally conflicted.

\FloatBarrier

\section{Confidence-Masked Conflict Stress Test}

A natural concern is that SEF may only rediscover ordinary model confidence. We therefore add a stress test built to separate these two ideas. The data are generated from a transparent signed-evidence system with one supporting driver, one opposing driver, weak extra signal, and nuisance variables. The fitted model can still be confident when its score is far from zero. However, when strong support and strong opposition arrive together, the label is deliberately less reliable. This creates the situation SEF is meant to diagnose: a prediction can look confident while the evidence behind it is mixed.

The base learner is logistic regression. SEF evidence is computed from the fitted logit, so this stress test does not give SEF information that is hidden from the model. We first select the $40\%$ most confident predictions. Inside that already-confident group, we compare the lowest and highest quartiles of SEF conflict.

\begin{table}[!htbp]
\centering
\caption{Confidence-masked conflict stress test. Even among predictions that already look confident, high SEF conflict marks a riskier subset.}
\label{tab:confidence-masking}
\begin{tabular}{lrr}
\toprule
Group & Cases & Error \\
\midrule
All test cases & 21000 & 0.277 \\
High confidence & 8400 & 0.111 \\
High confidence, low SEF conflict & 2100 & 0.071 \\
High confidence, high SEF conflict & 2100 & 0.174 \\
\bottomrule
\end{tabular}

\end{table}

\begin{figure}[!htbp]
\centering
\includegraphics[width=0.90\textwidth]{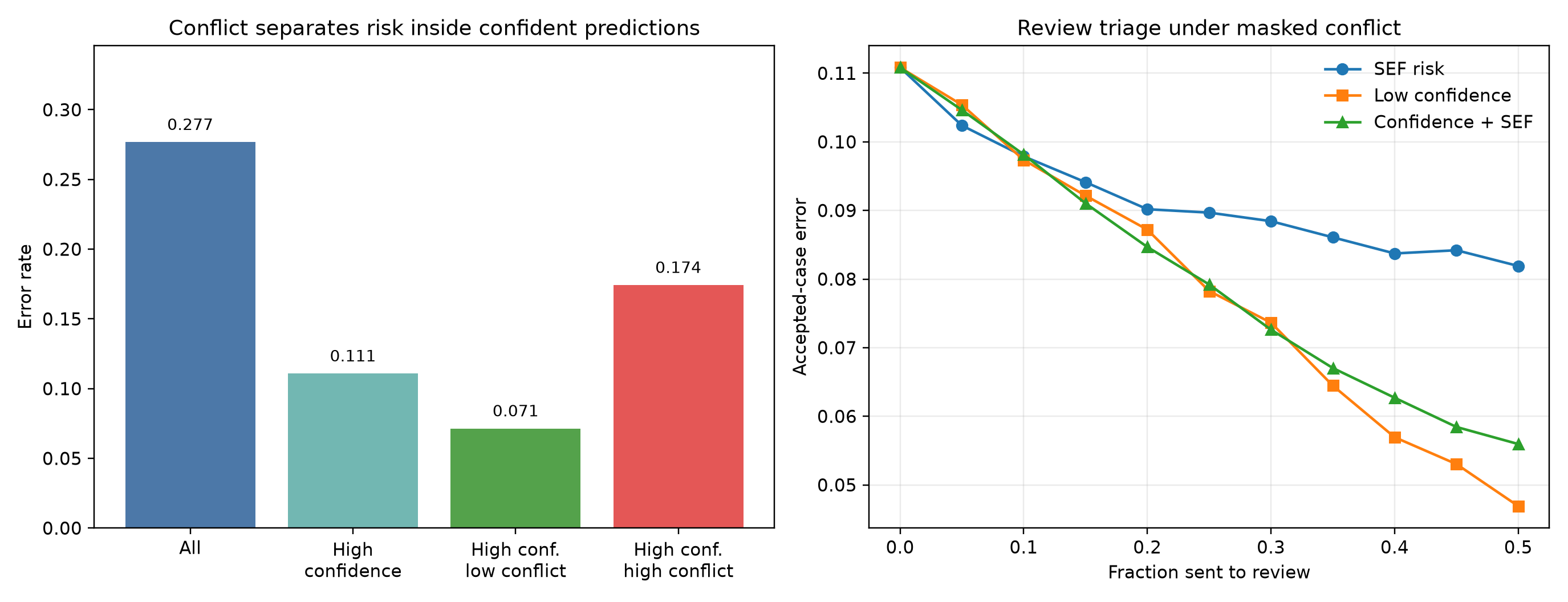}
\caption{Confidence-masked conflict stress test. Left: among high-confidence predictions, the high-conflict quartile has higher error than the low-conflict quartile. Right: when triage is restricted to high-confidence predictions, confidence remains a strong baseline, while combining confidence with SEF gives the lowest accepted-case error in this run.}
\label{fig:confidence-masking}
\end{figure}

Table~\ref{tab:confidence-masking} and Figure~\ref{fig:confidence-masking} present the stress-test results. The result is useful because it is modest and specific. The overall error is $0.277$, and the high-confidence group has error $0.111$. But within that high-confidence group, low SEF conflict has error $0.071$, while high SEF conflict has error $0.174$. Thus, SEF conflict separates a riskier subset even after confidence has already selected cases that look safe.

The triage curve reinforces this. At a $30\%$ review budget inside the high-confidence group, SEF-only triage gives accepted-case error $0.088$, confidence-only triage gives $0.074$, and the combined confidence-plus-SEF rule gives $0.073$. Combining the two scores achieves the lowest accepted-case error.

\FloatBarrier

\section{Real Healthcare Benchmarks Beyond Confidence}

We next test the same idea on real healthcare data. This benchmark uses three public OpenML data sets: Pima diabetes, Statlog heart disease, and blood-transfusion donation behavior. The first two are direct health-risk tasks. The third is a health-service prediction task. These data sets are smaller than the Covertype benchmark, but they are useful because they are closer to the kind of setting where a user may ask whether a prediction is safe enough to trust.

The question here is deliberately narrow:

\begin{quote}
Among predictions that already look confident, does SEF conflict still separate safer cases from riskier cases?
\end{quote}

For each data set, we run 50 stratified train-test splits. The base model is a balanced logistic regression with standard preprocessing. We first keep the top half of test predictions by model confidence. Inside that already-confident group, we compare the lowest and highest quartiles of SEF conflict.

\begin{table}[!htbp]
\centering
\caption{Healthcare benchmark beyond confidence. Within high-confidence predictions, high SEF conflict has higher error than low SEF conflict across all three real data sets. Split success is the fraction of 50 train-test splits where the high-conflict error exceeds the low-conflict error.}
\label{tab:healthcare-beyond-confidence}
\resizebox{\textwidth}{!}{\begin{tabular}{lrrrrr}
\toprule
Dataset & $n$ & Error & High-conf. low-conflict & High-conf. high-conflict & Split success \\
\midrule
Diabetes & 768 & 0.247 & 0.050 & 0.177 & 1.00 \\
Heart disease & 270 & 0.167 & 0.010 & 0.143 & 0.86 \\
Blood transfusion & 748 & 0.339 & 0.162 & 0.244 & 0.88 \\
\bottomrule
\end{tabular}
}
\end{table}

\begin{figure}[!htbp]
\centering
\includegraphics[width=0.90\textwidth]{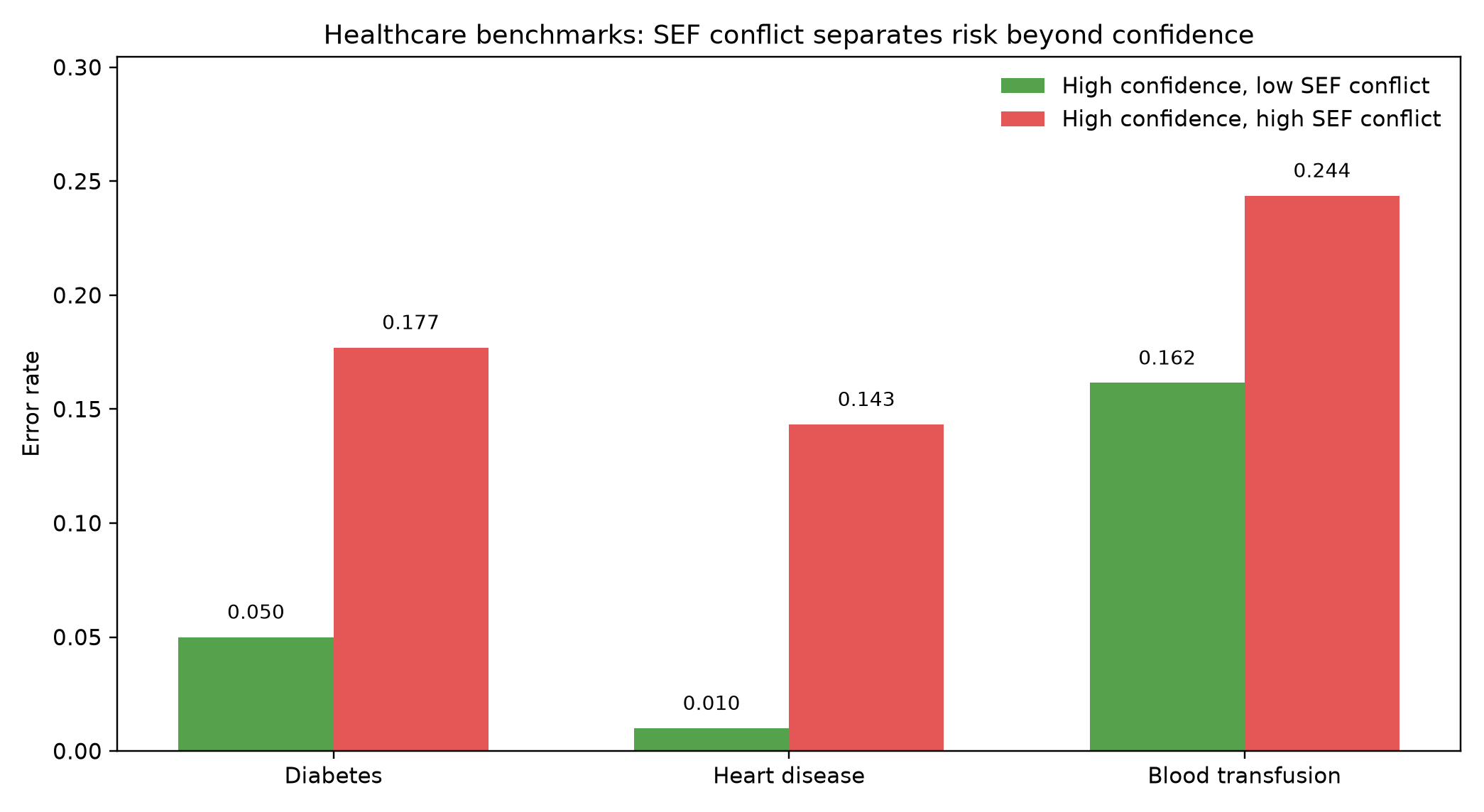}
\caption{Healthcare benchmarks beyond confidence. After restricting to predictions that already look confident, SEF conflict still separates lower-risk and higher-risk subsets.}
\label{fig:healthcare-beyond-confidence}
\end{figure}

Table~\ref{tab:healthcare-beyond-confidence} and Figure~\ref{fig:healthcare-beyond-confidence} show that the pattern is strong enough to matter. In the diabetes data, high-confidence low-conflict cases have average error $0.050$, while high-confidence high-conflict cases have average error $0.177$. In the heart-disease data, the same comparison is $0.010$ versus $0.143$. In the blood-transfusion data, it is $0.162$ versus $0.244$. The high-conflict group has higher error in all 50 diabetes splits, in $86\%$ of the heart-disease splits, and in $88\%$ of the blood-transfusion splits.

This is the clearest real-data evidence that SEF adds something beyond ordinary confidence. The result does not say that confidence is unimportant. It says that, even after confidence has selected predictions that look safe, evidence conflict can still identify a subgroup that deserves more caution.

\FloatBarrier

\section{Comparison with Attribution Entropy}

A reviewer may ask whether SEF conflict is just another name for attribution entropy. This is a fair question. Attribution entropy measures how spread out the absolute attribution magnitudes are across features. SEF conflict measures a different object: how much signed evidence sits on the supporting side and how much sits on the opposing side. A case can have high attribution entropy because many features contribute in the same direction. A case can also have low entropy but high conflict if a few large features push against each other.

To test this directly, we compare four ranking scores on the healthcare benchmark suite: SEF conflict, attribution entropy, a Gini-style attribution spread score, and low model confidence. Each score is evaluated inside the high-confidence group used in the previous section. For each method, we report the error rate among the top-risk quartile and the AUC for ranking errors.

\begin{table}[!htbp]
\centering
\caption{Comparison with attribution entropy and Gini-style attribution spread. Values are averaged over 50 train-test splits and are computed only within predictions that already look confident. Higher top-quartile error and higher error AUC mean the score is better at ranking risky cases.}
\label{tab:entropy-comparison}
\resizebox{\textwidth}{!}{\begin{tabular}{llrr}
\toprule
Dataset & Ranking score & Top-quartile error & Error AUC \\
\midrule
Diabetes & SEF conflict & 0.177 & 0.640 \\
Diabetes & Attribution entropy & 0.158 & 0.524 \\
Diabetes & Attribution Gini spread & 0.138 & 0.505 \\
Diabetes & Low confidence & 0.189 & 0.591 \\
Heart disease & SEF conflict & 0.143 & 0.729 \\
Heart disease & Attribution entropy & 0.032 & 0.384 \\
Heart disease & Attribution Gini spread & 0.033 & 0.391 \\
Heart disease & Low confidence & 0.117 & 0.692 \\
Blood transfusion & SEF conflict & 0.244 & 0.571 \\
Blood transfusion & Attribution entropy & 0.234 & 0.580 \\
Blood transfusion & Attribution Gini spread & 0.238 & 0.590 \\
Blood transfusion & Low confidence & 0.329 & 0.664 \\
\bottomrule
\end{tabular}
}
\end{table}

\begin{figure}[!htbp]
\centering
\includegraphics[width=0.90\textwidth]{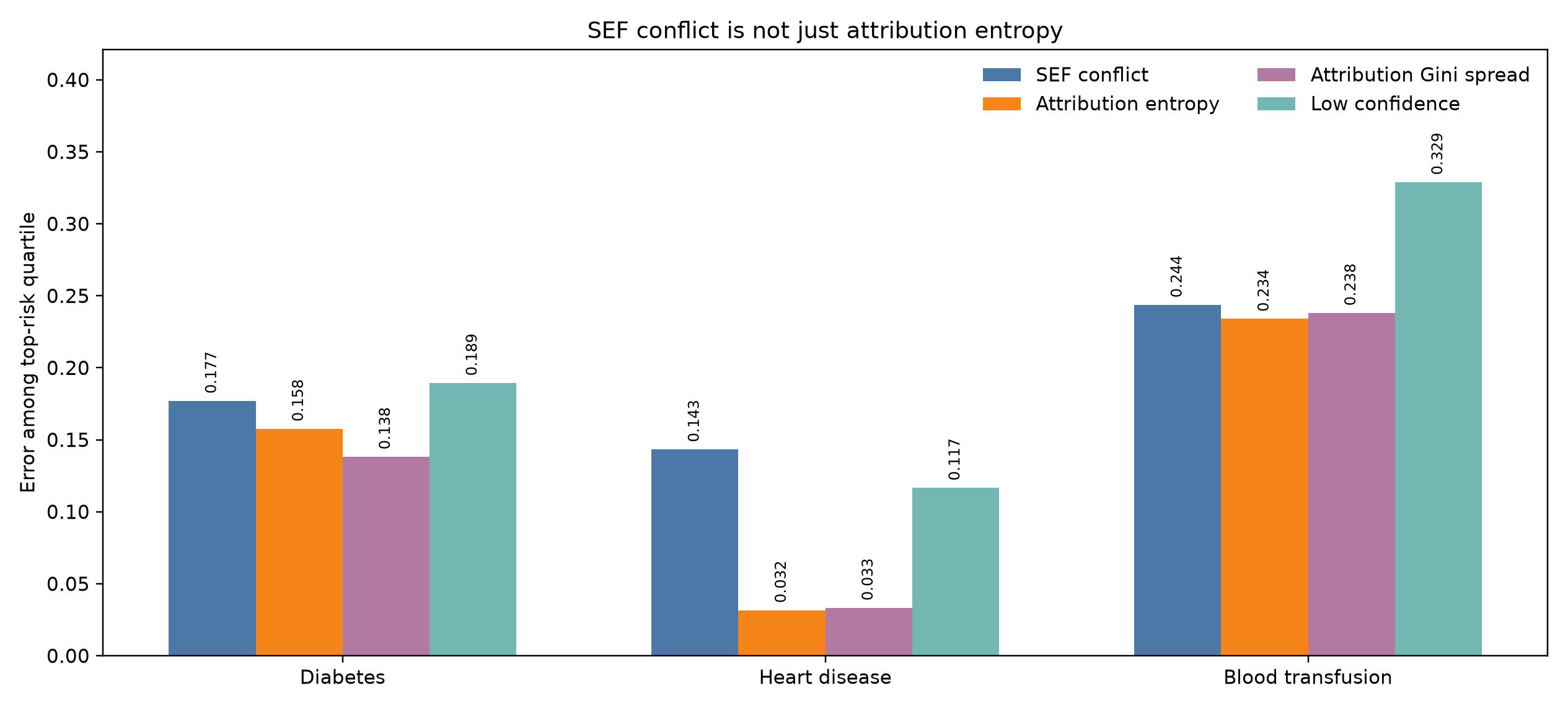}
\caption{SEF conflict compared with attribution entropy, attribution Gini spread, and low confidence. SEF conflict is not a monotone copy of attribution entropy. It is especially stronger than entropy and Gini spread on the diabetes and heart-disease tasks.}
\label{fig:entropy-comparison}
\end{figure}

Table~\ref{tab:entropy-comparison} and Figure~\ref{fig:entropy-comparison} present the comparison, which gives a useful answer. On diabetes, SEF conflict has error AUC $0.640$, while attribution entropy has error AUC $0.524$ and Gini spread has error AUC $0.505$. On heart disease, SEF conflict has error AUC $0.729$, while entropy has $0.384$ and Gini spread has $0.391$. On blood transfusion, the result is less favorable to SEF: low confidence is the strongest ranking score, and entropy, Gini spread, and SEF conflict are closer together.

The rank correlations also show that SEF is not just entropy in disguise. The average Spearman correlation between SEF conflict and attribution entropy is $0.119$ on diabetes, $0.083$ on heart disease, and $0.603$ on blood transfusion. Thus, in two of the three healthcare tasks, the two rankings are quite different. In the task where they are more similar, SEF does not dominate. One likely reason is task structure. In the blood-transfusion data, prediction errors are more closely tied to being near the decision boundary, so confidence already captures much of the risk. In diabetes and heart disease, riskier high-confidence cases are more often cases where strong positive and negative evidence coexist. This is the right conclusion: SEF is not a universal replacement for entropy or confidence, but it captures a signed-evidence feature that entropy alone can miss.

\FloatBarrier

\section{External Finance Stress Test and Scope Boundary}

The earlier experiments show several settings where SEF conflict separates risk beyond confidence. A serious method also needs a test that can fail. We therefore ran a held-out stress test on three public OpenML data sets that were not used to shape the original examples: German Credit, Bank Marketing, and Default of Credit Card Clients. Each experiment uses 25 stratified train-test splits and the same balanced logistic-model evidence construction used in the healthcare study.

\begin{table}[!htbp]
\centering
\caption{External finance and credit-risk stress test over 25 splits. The comparison is restricted to the most confident half of predictions. Success is the fraction of splits where the high-conflict quartile has higher error than the low-conflict quartile.}
\label{tab:finance-scope}
\resizebox{\textwidth}{!}{\begin{tabular}{lrrrrrr}
\toprule
Dataset & $n$ & Error & High-conf. low-conflict & High-conf. high-conflict & Success & SEF AUC \\
\midrule
Bank marketing & 45211 & 0.156 & 0.176 & 0.007 & 0.00 & 0.204 \\
Credit-card default & 30000 & 0.311 & 0.296 & 0.167 & 0.00 & 0.423 \\
German credit & 1000 & 0.283 & 0.125 & 0.160 & 0.56 & 0.530 \\
\bottomrule
\end{tabular}
}
\end{table}

Table~\ref{tab:finance-scope} reports the finance results. The result draws a clear boundary around the method. On German Credit, the high-conflict group has somewhat higher error than the low-conflict group, but the effect is modest and appears in only $56\%$ of splits. On Bank Marketing and Credit Card Default, the pattern reverses: among already-confident predictions, low-conflict cases are more error-prone. This is not a minor exception. It means conflict is not a universal error score, and a naive high-conflict review rule can make triage worse.

A plausible mechanism is coherent misspecification. In the healthcare and Covertype examples, errors often appear when strong evidence points in competing directions. In the two larger finance tasks, many confident errors appear to be low-conflict cases where the fitted model has a coherent but wrong story. The external replication shows that this reversal is not specific to finance: it also appears in Adult Income, Mammography, and the PC1 and KC1 software-defect tasks. Across these domains, class imbalance, correlated predictors, omitted structure, or systematic model bias can produce a confident and internally consistent explanation that is still wrong. We do not treat this mechanism as proved. Instead, the next analysis makes the operational claim testable: it asks whether conflict adds information after confidence and attribution entropy are known, and separately checks whether the sign of that information supports high-conflict review.

\section{ScopeGate: Conditional Value and a Finite-Sample Deployment Test}

The earlier quartile comparisons ask whether high-conflict and low-conflict groups have different errors. They do not establish that conflict adds information after confidence and attribution entropy are already available. We therefore fit two error-risk models only within the most confident half of held-out predictions. The base model uses low confidence and exact linear-attribution entropy. The augmented model adds SEF conflict. Predictions from both error-risk models are cross-fitted, so a case is never scored by an error model trained on that case. We repeat the comparison over 15 train-test splits.

For the balanced logistic models used here, the evidence terms $(z_j-\bar z_j)\widehat\beta_j$ are the exact interventional linear-SHAP contributions on the logit scale under the empirical mean reference. This makes the entropy comparison a SHAP-based comparison for these models. It does not claim conditional-SHAP validity when transformed features are dependent.

\begin{table}[!htbp]
\centering
\caption{ScopeGate analysis on held-out high-confidence predictions. Base AUC cross-fits error risk from low confidence and attribution entropy; $+$SEF AUC also includes conflict. Gap is high-conflict minus low-conflict error inside the ScopeGate cross-fitting split, so it need not equal the simpler quartile gap reported in earlier benchmark tables. Eligible/pass reports the fraction of splits meeting the information rule and, among eligible splits, the fraction passing the one-sided positive-association diagnostic. Time excludes model fitting and bootstrap stability.}
\label{tab:scope-diagnostic}
\resizebox{\textwidth}{!}{\begin{tabular}{lrrrrrr}
\toprule
Dataset & Base AUC & +SEF AUC & $\Delta$AUC & Gap & Eligible/pass & Time/1k rows \\
\midrule
Diabetes & 0.488 & 0.523 & +0.035 & +0.104 & 0.20/0.00 & 50.4 ms \\
Heart disease & 0.663 & 0.664 & +0.001 & +0.144 & 0.00/-- & 169.3 ms \\
Blood transfusion & 0.629 & 0.619 & -0.011 & +0.073 & 1.00/0.20 & 48.5 ms \\
German credit & 0.636 & 0.665 & +0.030 & +0.033 & 1.00/0.00 & 118.2 ms \\
Bank marketing & 0.853 & 0.922 & +0.069 & -0.172 & 1.00/0.00 & 15.4 ms \\
Credit-card default & 0.562 & 0.653 & +0.091 & -0.129 & 1.00/0.00 & 3.5 ms \\
\bottomrule
\end{tabular}
}
\end{table}

\begin{figure}[!htbp]
\centering
\includegraphics[width=0.94\textwidth]{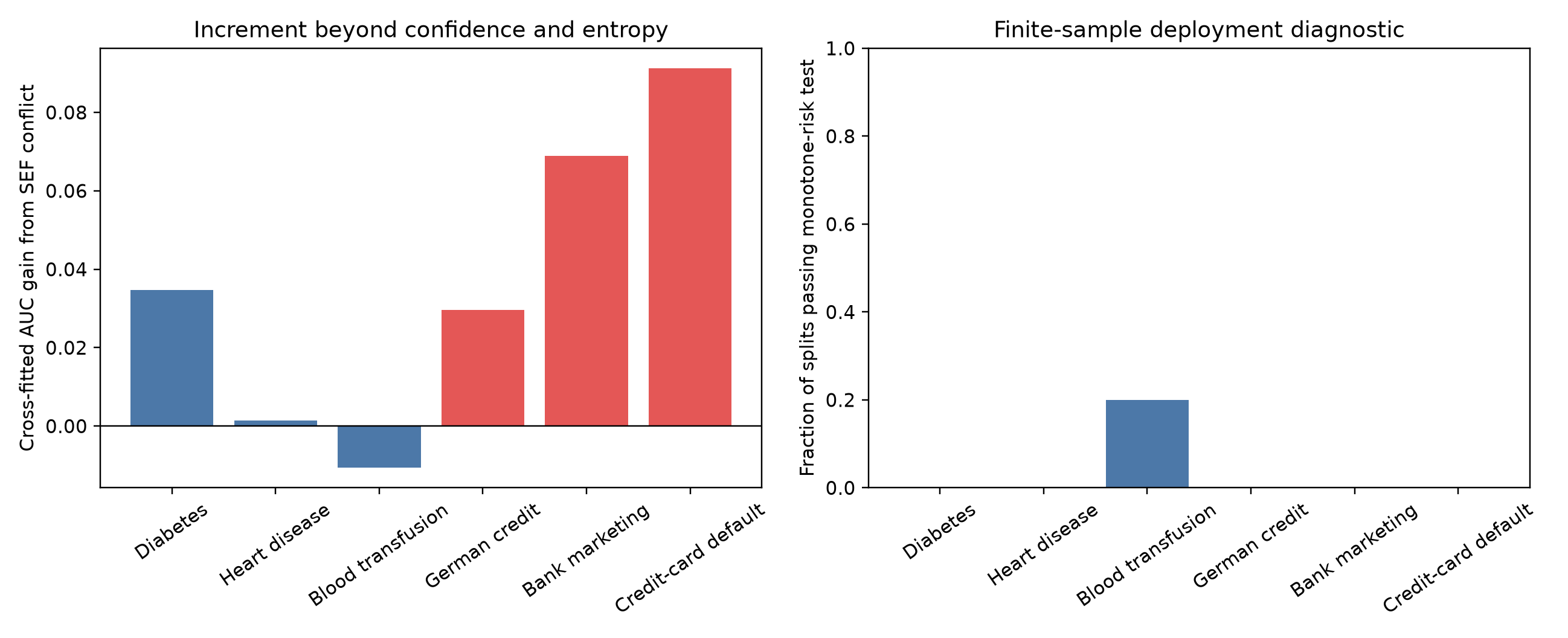}
\caption{Left: cross-fitted error-ranking gain after adding SEF conflict to confidence and attribution entropy. Right: the finite-sample positive-direction diagnostic. Red finance bars show why incremental information and review direction must be treated as separate questions.}
\label{fig:scope-diagnostic}
\end{figure}

Figure~\ref{fig:scope-diagnostic} visualizes the cross-fitted gains and direction diagnostics. The stricter analysis changes the interpretation of the paper. Conflict adds cross-fitted error-ranking AUC beyond confidence and entropy on Diabetes ($+0.035$), German Credit ($+0.030$), Bank Marketing ($+0.069$), and Credit Card Default ($+0.091$). Paired sign-flip tests across splits give $p=0.0731$, $0.0037$, $0.0002$, and $0.0002$, respectively. The Diabetes gain is suggestive but does not reject at the conventional $5\%$ level, which agrees with its low ScopeGate eligibility and zero pass rate. The gain is negligible on Heart Disease ($+0.001$) and negative on Blood Transfusion ($-0.011$). Thus, conditional value is strong on the two large finance tasks and mixed on the smaller healthcare tasks.

Direction is a separate issue. Bank Marketing and Credit Card Default have negative conflict gaps, $-0.172$ and $-0.129$, and never pass the positive-direction diagnostic. Their large AUC gains mean that conflict contains useful information, but the fitted error model learns that information in the opposite direction. German Credit has a small positive average gap but no split passes the $5\%$ diagnostic. Diabetes and Heart Disease often have too few high-confidence errors for a decisive calibration test; Heart Disease is ineligible in every split under the stated information rule. This eligibility failure is itself useful information: when there are too few errors or too little support in the conflict tails, the method should not pretend to certify a review policy. The practical conclusion is sharper than the earlier speculative explanation: SEF conflict can add information without supporting a universal high-conflict review rule.

The same audit measures computational cost. Across these six data sets, exact linear attributions plus SEF summaries require between 3.5 and 169.3 milliseconds per 1,000 held-out rows on the test machine, excluding model fitting. Fixed preprocessing overhead makes the smallest data sets slower on a per-row basis. The SEF aggregation itself takes at most about 18 milliseconds per split in this study. We also timed bootstrap stability separately on the four standard tasks. With $B=40$ refits, stability estimation took 0.188--0.257 seconds, compared with 0.0054--0.0072 seconds for one fit-and-score pass, a measured ratio of 30.1--39.2 on the test machine. These wall-clock values are hardware- and model-dependent, but they show why one-pass SEF timing should not be read as total deployment cost.

\section{Independent External Replication}

The ScopeGate result should not rest only on the data sets used above. We therefore ran a separate replication suite on ten additional public OpenML binary-classification data sets: Adult Income, Spambase, Phoneme, Mammography, Credit Approval, QSAR Biodegradation, PC1 software defects, KC1 software defects, Electricity, and Magic Telescope. Together these data sets contain 139,325 observations and cover socioeconomic, communication, speech, healthcare, credit, chemistry, software, market, and physics tasks. None of these data sets was used to tune the earlier examples.

The replication uses the same pipeline as ScopeGate. For each data set, we run ten stratified train-test splits. Within the most confident half of held-out predictions, the base error-risk model uses low confidence and attribution entropy. The augmented model adds SEF conflict. The deployment diagnostic again asks a separate question: whether higher conflict is positively associated with observed error strongly enough to justify high-conflict review.

\begin{table}[!htbp]
\centering
\caption{Independent external replication on ten additional public OpenML data sets. Base AUC cross-fits error risk from low confidence and attribution entropy; $+$SEF AUC also includes conflict. Gap is high-conflict minus low-conflict error. Eligible/pass has the same meaning as in Table~\ref{tab:scope-diagnostic}.}
\label{tab:external-replication}
\resizebox{\textwidth}{!}{\begin{tabular}{lrrrrrr}
\toprule
Dataset & $n$ & Base AUC & +SEF AUC & $\Delta$AUC & Gap & Eligible/pass \\
\midrule
Adult income & 48842 & 0.783 & 0.839 & +0.056 & -0.002 & 1.00/0.00 \\
Spambase & 4601 & 0.547 & 0.518 & -0.029 & +0.003 & 0.00/-- \\
Phoneme & 5404 & 0.740 & 0.754 & +0.013 & +0.002 & 1.00/0.00 \\
Mammography & 11183 & 0.794 & 0.859 & +0.065 & -0.063 & 1.00/0.00 \\
Credit approval & 690 & 0.570 & 0.594 & +0.024 & +0.016 & 0.00/-- \\
QSAR biodeg. & 1055 & 0.609 & 0.615 & +0.006 & +0.034 & 0.00/-- \\
PC1 defects & 1109 & 0.658 & 0.692 & +0.035 & -0.196 & 1.00/0.00 \\
KC1 defects & 2109 & 0.727 & 0.728 & +0.001 & -0.012 & 1.00/0.00 \\
Electricity & 45312 & 0.651 & 0.651 & -0.000 & +0.083 & 1.00/1.00 \\
Magic telescope & 19020 & 0.586 & 0.730 & +0.144 & +0.229 & 1.00/1.00 \\
\bottomrule
\end{tabular}
}
\end{table}

\begin{figure}[!htbp]
\centering
\includegraphics[width=0.95\textwidth]{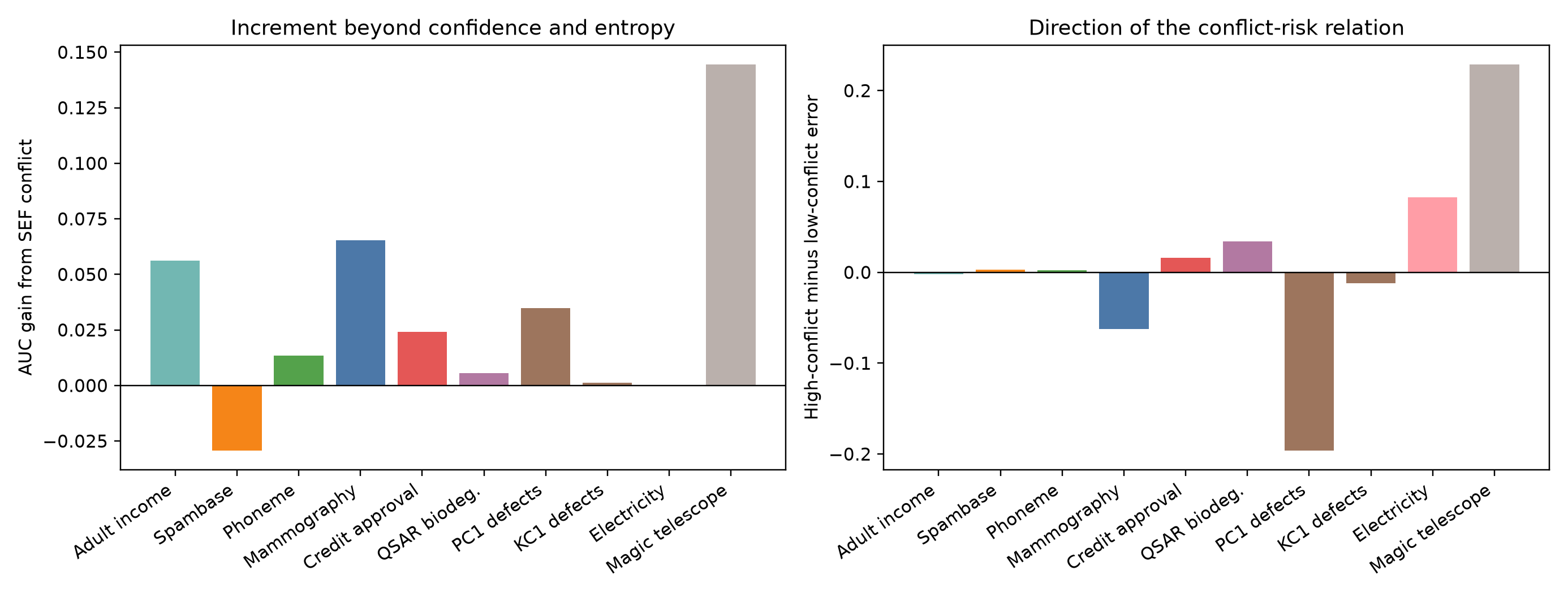}
\caption{Independent external replication. Left: incremental error-ranking value from adding SEF conflict after confidence and attribution entropy. Right: direction of the conflict-risk relation. The two panels show why information value and safe review direction are different questions.}
\label{fig:external-replication}
\end{figure}

Table~\ref{tab:external-replication} and Figure~\ref{fig:external-replication} present the external replication. The replication strengthens the paper, but it does so in a careful way. Conflict improves cross-fitted error-ranking AUC on seven of the ten external data sets. Electricity is effectively neutral: its displayed change rounds to $-0.000$, although its positive-direction diagnostic passes in every split. The largest gains appear on Magic Telescope ($+0.144$), Mammography ($+0.065$), Adult Income ($+0.056$), and PC1 defects ($+0.035$). This supports the claim that evidence conflict often contains information not captured by confidence and attribution entropy alone.

At the same time, the direction is not universal. Mammography, PC1 defects, KC1 defects, and Adult Income show positive incremental AUC but nonpositive or weak high-minus-low conflict gaps. In those tasks, conflict helps an error-risk model, but it should not be used as a simple high-conflict review rule. Electricity and Magic Telescope are the clearest positive-direction replications: both are eligible in every split and pass the one-sided ScopeGate diagnostic in every split. Spambase is a useful negative control because SEF does not add conditional AUC there, and the diagnostic is not eligible due to too few high-confidence errors.

This external replication moves the method from a set of examples to a broader empirical claim. The claim is not that high conflict is always bad. The claim is that signed evidence conflict is a measurable structure that often carries conditional error information, and ScopeGate is needed to decide whether that information points in the review direction a practitioner wants.

The external suite also gives a stronger test of whether SEF is merely attribution entropy under another name. Table~\ref{tab:external-entropy} compares the two scores, Gini spread, and low confidence inside the same held-out high-confidence groups. Figure~\ref{fig:external-entropy} shows both the error-ranking difference and the rank correlation between SEF conflict and entropy.

\begin{table}[!htbp]
\centering
\caption{Independent comparison with attribution entropy on ten external data sets. Error AUC measures how well each score ranks errors inside the held-out high-confidence group. The last column is the split-averaged Spearman correlation between the SEF and entropy rankings.}
\label{tab:external-entropy}
\resizebox{\textwidth}{!}{\begin{tabular}{lrrrrr}
\toprule
Dataset & SEF AUC & Entropy AUC & Gini AUC & Low-conf. AUC & $\rho$(SEF, entropy) \\
\midrule
Adult income & 0.471 & 0.385 & 0.352 & 0.772 & 0.441 \\
Spambase & 0.552 & 0.414 & 0.390 & 0.631 & 0.294 \\
Phoneme & 0.489 & 0.678 & 0.653 & 0.685 & -0.101 \\
Mammography & 0.135 & 0.193 & 0.202 & 0.284 & 0.805 \\
Credit approval & 0.560 & 0.583 & 0.547 & 0.571 & 0.272 \\
QSAR biodeg. & 0.557 & 0.570 & 0.575 & 0.725 & -0.075 \\
PC1 defects & 0.322 & 0.295 & 0.315 & 0.487 & 0.182 \\
KC1 defects & 0.478 & 0.235 & 0.249 & 0.592 & 0.041 \\
Electricity & 0.582 & 0.597 & 0.577 & 0.648 & 0.079 \\
Magic telescope & 0.722 & 0.445 & 0.478 & 0.568 & 0.286 \\
\bottomrule
\end{tabular}
}
\end{table}

\begin{figure}[!htbp]
\centering
\includegraphics[width=0.96\textwidth]{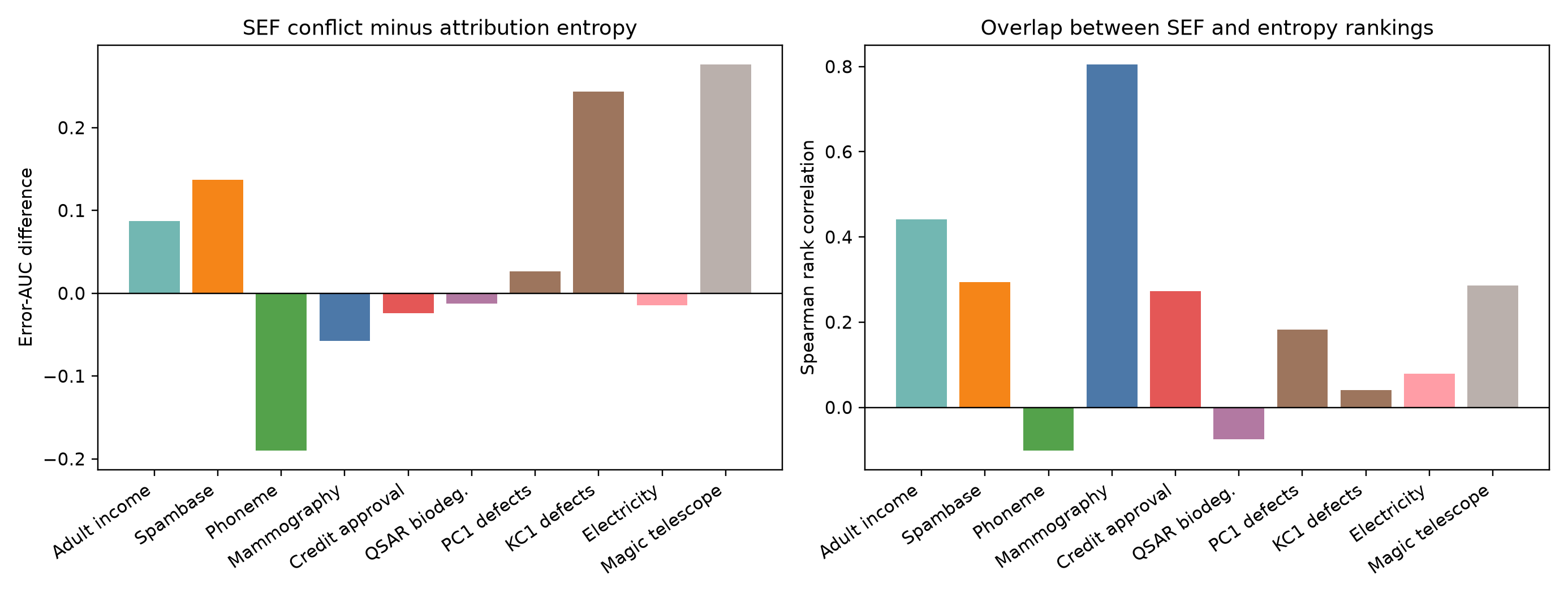}
\caption{Independent SEF-versus-entropy comparison. Left: error-AUC difference, where positive bars favor SEF conflict. Right: rank correlation between SEF conflict and attribution entropy. The mixed signs and correlations show that neither score is a monotone replacement for the other.}
\label{fig:external-entropy}
\end{figure}

The expanded comparison does not support a blanket dominance claim. SEF conflict ranks errors substantially better than entropy on Adult Income, Spambase, KC1 defects, and Magic Telescope, with the largest gain on Magic Telescope ($+0.276$). Entropy is clearly stronger on Phoneme ($-0.190$) and moderately stronger on Mammography ($-0.058$). Low confidence remains the strongest single raw ranking score on eight of the ten tasks; Magic Telescope is the clearest exception, where SEF conflict has AUC $0.722$ versus $0.568$ for low confidence. The SEF--entropy rank correlation ranges from $-0.101$ on Phoneme to $0.805$ on Mammography. These results answer the novelty question more directly than the three-data-set healthcare comparison: SEF and entropy can overlap strongly, differ sharply, or favor different rankings depending on the task. Their relationship is empirical and should be reported rather than assumed.

Figure~\ref{fig:evidence-frontier} turns the external ScopeGate result into a review-budget check. Each curve removes the highest-conflict fraction from the held-out high-confidence group and reports error among the remaining cases.

\begin{figure}[!htbp]
\centering
\includegraphics[width=0.94\textwidth]{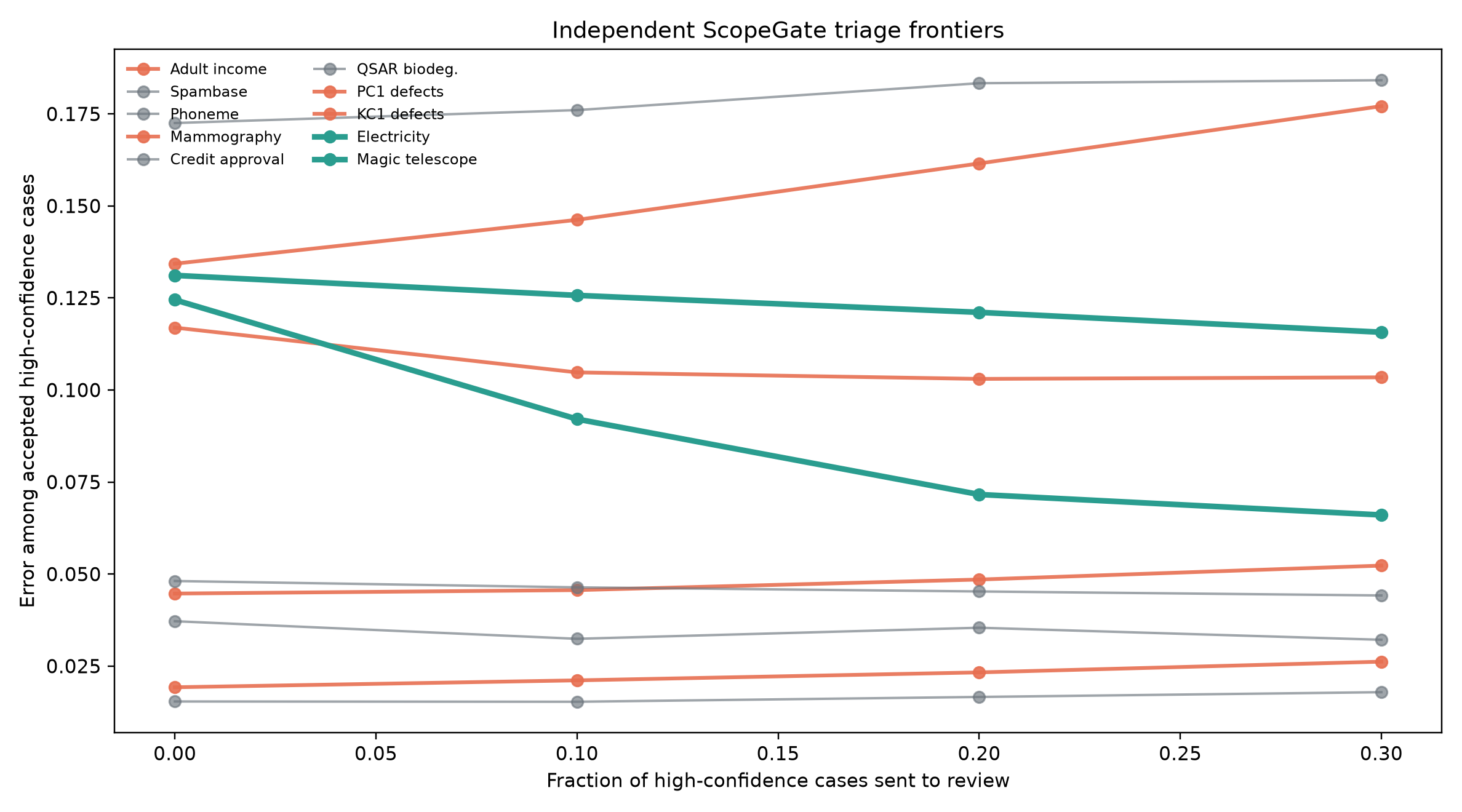}
\caption{Independent ScopeGate triage frontiers. Teal curves are Electricity and Magic Telescope, where the positive-direction diagnostic passes in every split. Orange curves mark identified direction reversals. Gray curves are inconclusive or unsupported settings. Reviewing more high-conflict cases lowers accepted-case error in the supported tasks but can raise it in reversal settings.}
\label{fig:evidence-frontier}
\end{figure}

The frontier makes the deployment boundary visible. On Magic Telescope, reviewing the highest-conflict $30\%$ lowers accepted high-confidence error from about $0.124$ to $0.066$. Electricity also improves, though more modestly. By contrast, accepted error rises on Adult Income and several reversal tasks when high-conflict cases are removed. A single pooled curve would hide this difference. ScopeGate is therefore not an optional refinement: it determines whether a high-conflict review rule is supported in the target population.

\FloatBarrier

\section{Comparison with Alternative Tools}

Table~\ref{tab:method-comparison} summarizes the role of SEF relative to common alternatives. The table is not meant to say that SEF dominates all existing methods. It shows the specific gap SEF is designed to fill.

\begin{table}[!htbp]
\centering
\caption{Conceptual comparison. SEF is positioned as an evidence-audit layer, not as a replacement for attribution, confidence, calibration, or selective prediction.}
\label{tab:method-comparison}
\small
\begin{tabular}{p{0.24\textwidth}p{0.66\textwidth}}
\toprule
Method family & What SEF adds \\
\midrule
Model confidence & Reports how decisive the model output is. SEF adds a check for mixed evidence behind that output. \\
SHAP and Shapley-style attribution & Reports feature contributions. SEF turns signed contributions into conflict, stability, and reliability scores. \\
LIME & Gives a local surrogate explanation. SEF adds an explicit reliability audit for that explanation. \\
Permutation or model reliance & Describes variable importance, often at a global level. SEF moves the question to case-level evidence conflict. \\
Conformal prediction & Gives calibrated prediction or risk guarantees. SEF adds evidence decomposition before reliability screening. \\
Selective prediction & Decides which cases to accept or reject. SEF ranks review cases by evidence risk, and Table~\ref{tab:selective-baseline} compares it with confidence and a learned reject baseline. \\
SEF & Reports support, opposition, conflict, stability, and a reliable evidence score for each audited prediction. \\
\bottomrule
\end{tabular}
\end{table}

\FloatBarrier

\section{Related Work and Positioning}

SEF is related to several active areas, but it is not the same as any one of them.

\subsection{Interpretable and Explainable Learning}

Interpretable machine learning has grown into a large field, with work on what explanations should mean, how they should be evaluated, and when a transparent model should be preferred to a black-box model \citep{doshi2017rigorous,murdoch2019definitions,rudin2019stop}. SEF fits into this discussion by taking a narrow but practical view. It does not try to solve every problem in interpretability. It asks whether the evidence behind one prediction is one-sided, conflicted, or unstable.

This is important because a prediction can look simple from the outside while being mixed on the inside. A model may give a high probability, but that high probability may come from strong supporting evidence fighting against strong opposing evidence. SEF makes that internal disagreement visible.

\subsection{Local Attribution Methods}

SEF uses signed evidence terms, so it is naturally connected to attribution methods such as Shapley values, LIME, and integrated gradients \citep{shapley1953value,lundberg2017unified,ribeiro2016why,sundararajan2017axiomatic}. Those methods answer a useful question: which features contributed to the prediction?

SEF starts after that question. Once signed contributions are available, SEF asks how those contributions behave as an evidence system. It separates total support from total opposition, measures conflict, estimates perturbation stability, and builds a reliability score. In this sense, SEF can use existing attribution tools as inputs, but its main object is not attribution size. Its main object is evidence reliability.

\subsection{Global Effects and Model Importance}

Other work studies global feature effects and model reliance. Partial dependence, accumulated local effects, and model-class reliance are useful for understanding how variables matter across a population \citep{apley2020visualizing,fisher2019all}. SEF has a different target. It is mainly instance-level. It studies whether one prediction is supported by clean evidence or by a fragile balance of competing signals.

These two views are complementary. Global effect tools can say that a variable matters on average. SEF can say that, for this particular case, the evidence is strong, conflicted, stable, or unstable.

\subsection{Stability and Conformal Reliability}

The stability part of SEF is connected to the statistical idea that a reported signal should survive reasonable perturbations of the data or fitting process \citep{meinshausen2010stability}. The conformal screen is connected to distribution-free prediction and risk control \citep{shafer2008tutorial,angelopoulos2021gentle,bates2021distribution}. SEF brings these ideas into the explanation setting by asking whether the explanation itself is stable enough to trust.

This matters in applied data analysis. A user may not only ask whether a prediction is accurate on average. The user may ask whether the evidence pattern behind a particular prediction would remain similar if the data were resampled, slightly perturbed, or refit.

\subsection{Selective Prediction and Review Triage}

SEF-Audit is also related to selective prediction and reject-option learning, where a system can abstain or send cases to review \citep{geifman2017selective,geifman2019selectivenet}. The difference is the review score. Many selective systems rely on confidence or train a reject head. SEF-Audit ranks cases by evidence risk: high conflict, low stability, or both.

The experiments show why this distinction matters. Confidence is a strong baseline on clean benchmark tasks, and we do not hide that. SEF adds a different reason for review. It can flag a case because the evidence is internally conflicted, even when the model output looks decisive. The best use case is not replacing selective prediction. It is adding an evidence reason to a selective-prediction workflow.

\subsection{Position of SEF}

SEF is best understood as a model-agnostic evidence-audit layer. It is not a new base learner like a random forest or a gradient boosting machine \citep{breiman2001random,friedman2001greedy}. It is not a replacement for SHAP, LIME, or integrated gradients. It is a way to take signed evidence from a fitted model and turn it into support, opposition, conflict, stability, and review risk.

\section{Reproducibility}

All experiments are implemented as standalone Python scripts. The replication archive includes a single runner, \texttt{scripts/run\_core\_experiments.py}, that regenerates the main result tables and figures. The shared SEF implementation uses $\varepsilon=10^{-12}$ for numerical stability in all reported experiments; the algebraic identities in the theory are stated for $\varepsilon=0$ to show the clean population form.

The seed schedule is fixed in the scripts. Standard benchmark, selective-baseline, ablation, conformal-screen, healthcare, entropy-comparison, and identification stress-test experiments use 50 repetitions. The identification stress test uses deterministic offsets from seed \texttt{20260618}; the other repeated studies in this group use seeds $0,\ldots,49$. The black-box model-agnostic experiment uses seeds $0,\ldots,19$. The finance stress test uses seeds $0,\ldots,24$. ScopeGate uses seeds $0,\ldots,14$ plus seed \texttt{20260618} for the paired sign-flip test. The independent external replication uses seeds $0,\ldots,9$. The synthetic and Covertype sampling steps use seed \texttt{20260617}; Covertype bootstrap intervals use deterministic offsets from that seed. The multi-class standard tasks use seeds $0,\ldots,29$, and the multi-class Covertype splits use seeds $0,\ldots,4$ after drawing the working sample.

\begin{table}[!htbp]
\centering
\caption{Public data sources and identifiers used in the experiments.}
\label{tab:data-sources}
\resizebox{\textwidth}{!}{%
\begin{tabular}{lll}
\toprule
Study & Source & Dataset identifier \\
\midrule
Standard benchmarks & scikit-learn & Iris, Wine, Breast Cancer, Digits \\
Large benchmark & UCI via scikit-learn & Covertype \\
Healthcare & OpenML & diabetes v1; heart-statlog v1; blood-transfusion-service-center v1 \\
Finance stress test & OpenML & credit-g v1; bank-marketing v1; default-of-credit-card-clients v1 \\
External replication & OpenML & adult v2; spambase v1; phoneme v1; mammography v1; credit-approval v1; qsar-biodeg v1; pc1 v1; kc1 v1; electricity v1; MagicTelescope v1 \\
\bottomrule
\end{tabular}
}
\end{table}

Table~\ref{tab:data-sources} lists the public data sources used in all experiments. The replication archive also records the raw run-level results, summary tables, dependency list, model settings, OpenML names and versions, and the measured stability-runtime audit in \texttt{results/sef\_stability\_runtime.csv}. No private data are used. The binary and multi-class Covertype studies are drawn from the same public source and may overlap at the row level; they should be read as complementary stress tests of different prediction tasks, not as independent confirmatory samples. The complete executable archive is publicly available at \url{https://github.com/jopoku16/sef-scopegate}.

\FloatBarrier

\section{Practical Guidance}

This section gives a simple guide for using SEF in practice.

\textbf{Use SEF when individual predictions need review.} SEF is most useful when a prediction may trigger a decision, a report, a diagnosis, a risk score, or a human review. In those settings, the user often needs to know more than the predicted class or probability.

\textbf{Choose the reference distribution carefully.} The reference distribution defines what counts as baseline evidence. For a general model audit, the training distribution is a natural starting point. For subgroup analysis, a subgroup reference may be better. For policy or clinical work, the reference should match the population to which the prediction will be compared.

\textbf{Use grouped evidence when features are strongly related.} If features are highly correlated, replacing one feature at a time can create unrealistic cases. In that setting, SEF should be computed on feature groups or with a conditional replacement rule.

\textbf{Read conflict and confidence together.} A high-confidence prediction is not always a clean prediction. SEF conflict asks whether strong support and strong opposition are both present. Confidence asks how decisive the model output is. These are different questions, and the experiments show that both can be useful.

\textbf{Calibrate before using SEF for review triage.} Compute the proposed review score on a separate calibration sample with observed outcomes. Activate high-SEF-risk review only when the positive-direction permutation diagnostic is adequately informed and rejects at the chosen level. If the estimated direction is negative, if the test does not reject, or if too few errors are observed, do not use high SEF risk as an automatic review rule. Conflict may still be used as an input to a separately trained and cross-validated error model, but its direction must be learned rather than assumed.

\textbf{Plan healthcare validation around informative errors.} The present Diabetes and Heart Disease samples are too small for ScopeGate to support a high-conflict review rule reliably. A healthcare deployment study should specify the smallest clinically relevant conflict--error gap, collect an independent prospective calibration cohort, and continue collecting labeled outcomes until both conflict quartiles contain enough cases and the high-confidence group contains at least 20 errors and 20 correct predictions. The diagnostic should also be repeated by clinically important subgroup. Until those conditions are met, SEF may describe evidence structure, but it should not determine patient-level review decisions.

\textbf{Report the full setup.} A SEF analysis should state the fitted model, the reference distribution, the evidence construction, the perturbation scheme, the number of perturbations, and the review budget. Without these details, the reliability scores are hard to reproduce.

\textbf{Budget stability as repeated fitting, not as a one-pass score.} If stability uses $B$ model refits, its leading cost is approximately $B$ fits plus $B$ evidence calculations. In the standard-task audit, $B=40$ cost 30.1--39.2 times one fit-and-score pass. Larger models should measure this cost directly and consider fewer perturbations, grouped evidence units, parallel refits, or a cheaper perturbation scheme.

\section{Extensions Beyond Tabular Data}

The current experiments use tabular features, but the SEF framework applies whenever signed evidence units can be defined.

\textbf{Text.} For text classification, evidence units can be token-level or phrase-level attributions from integrated gradients, attention-based methods, or SHAP on token embeddings. A sentiment prediction might assign positive evidence to one phrase and negative evidence to another. SEF conflict then measures how much the prediction rests on competing textual signals, and the flip-margin theorem bounds how much attribution noise is needed to reverse the decision.

\textbf{Images.} For image classification, evidence units can be superpixel attributions from LIME or region attributions from gradient-based methods. A medical image classifier might assign positive evidence to a lesion region and negative evidence to a surrounding artifact. SEF conflict would flag predictions that depend on this internal disagreement, and the perturbation stability score would test whether the evidence pattern survives resampling of the attribution baseline.

\textbf{Retrieval-augmented systems.} For retrieval-augmented generation, evidence units can be retrieval-chunk attributions. If one retrieved passage supports an answer and another contradicts it, SEF conflict quantifies this disagreement before the system reports a confident answer. The multi-class extension applies naturally when the system must choose among several candidate answers.

In each case, the algebraic identities, the identification theorem, the flip-margin bound, and the perturbation stability concentration all transfer directly because they depend only on the existence of signed evidence terms, not on the data modality. The positive screening proposition and ScopeGate diagnostic also apply without modification. Full empirical evaluation of these extensions is future work.

Foundation-model and agentic settings raise a further difficulty that tabular data does not. There the inputs themselves move: prompts are rewritten, domains shift, and a single answer may depend on a chain of retrieval steps and tool calls rather than one prediction. A calibration layer for such systems has to control risk over the whole trajectory, not over one input-output pair. In separate work we develop conformal risk control for these two settings, one for prompt and domain shift in foundation models \citep{opoku2026promptshift} and one for retrieval and tool-use drift in agentic systems, where risk is scored over the whole trajectory of retrieval steps and tool calls rather than the final answer alone \citep{opoku2026toolchain}. SEF is complementary to that line of work rather than a substitute for it: conflict supplies a per-step evidence signal describing whether a given intermediate output rests on agreeing or competing support, which is the kind of quantity such a calibration layer can consume. Connecting the two is future work.

\section{Limitations}

SEF has several limitations.

First, the method depends on how evidence terms are constructed. Shapley-style evidence, linear-model evidence, and feature-replacement evidence can give different numerical values. This is not unique to SEF; it is a general issue in model explanation. In practice, the reference distribution and attribution construction should be reported clearly.

Second, feature replacement can be imperfect when variables are strongly correlated. Replacing one feature while holding all others fixed may create unrealistic feature combinations. Conditional replacement or grouped evidence terms may be better in those settings.

Third, SEF is not a causal method by itself. A positive evidence term means that a feature supports the model's prediction relative to the chosen reference. It does not mean that changing that feature would causally change the outcome unless additional causal assumptions are made.

Fourth, computational cost has two parts. The one-pass conflict calculation is modest in the linear studies: exact linear attributions plus SEF summaries take 3.5--169.3 milliseconds per 1,000 held-out rows on the test machine. Perturbation stability is much more expensive because it requires repeated refits or repeated evidence calculations. In our separate standard-task audit, $B=40$ stability refits cost 30.1--39.2 times one fit-and-score pass. This ratio is not a universal constant; larger applications may need fewer perturbations, grouped perturbations, parallel refits, or faster approximations.

Fifth, the current experiments focus on tabular prediction. We sketch extensions to text, images, and retrieval-augmented systems in the preceding section, and the algebraic results transfer directly, but the empirical evaluation of those extensions remains future work.

Sixth, the direction of conflict risk is task-dependent. Bank Marketing and Credit Card Default show strong incremental error-ranking information from conflict after confidence and entropy are known, but the association is negative. This rules out a universal claim that higher conflict always means higher error. The monotone evidence-risk condition must be checked in each target population before SEF risk is used as an automated review score.

Finally, SEF should be used alongside standard model checks. Accuracy, calibration, subgroup performance, confidence, and uncertainty measures still matter. SEF adds an evidence-reliability view; it does not remove the need for ordinary validation.

\section{Discussion}

This paper makes two separable contributions. The first is descriptive: SEF decomposes any fitted prediction into support, opposition, conflict, and stability, giving a structured view of evidence that confidence alone cannot provide. The identification theorem draws a sharp boundary around when this structure is recoverable from confidence and when it is not. The flip-margin theorem and perturbation radius give these quantities geometric and operational meaning.

The second contribution is prescriptive: ScopeGate tests whether the descriptive structure is safe to act on. The positive screening proposition gives the minimal condition under which SEF-Audit improves over uninformed review, and the permutation diagnostic checks that condition on held-out data. This two-stage design prevents the method from overpromising. The finance and external replication experiments show that evidence conflict can carry strong conditional information while pointing in the wrong direction for naive triage. ScopeGate catches this.

The practical upshot is that SEF is suited for settings where a prediction must be explained to a person who cares about reliability. In healthcare, credit, or public-policy applications, a model should show whether the evidence behind its prediction is one-sided, opposed, or unstable. SEF provides that description. Whether the description should trigger review is a separate, empirical question that ScopeGate answers for each target population.

\section*{Broader Impact}

SEF may help people inspect predictions in healthcare, finance, public policy,
and other high-stakes settings. Its main positive use is to show when the
evidence behind a prediction is mixed or unstable, so that a person can review
the case more carefully. The main risk is misuse: conflict could be treated as
a universal error score even though this paper shows that its direction can
reverse. SEF should therefore not decide access to care, credit, employment, or
other important outcomes by itself. Any review rule should be tested on a
separate, representative calibration sample, checked across relevant
subgroups, and used alongside ordinary accuracy, calibration, and fairness
checks.

\section*{Statements and Declarations}

\subsection*{Competing Interests}

The authors declare no competing financial or non-financial interests that are
directly or indirectly related to the work submitted for publication.

\subsection*{Funding}

The authors received no specific funding for this work.

\subsection*{Author Contributions}

J.O. contributed to conceptualisation, methodology, software, formal analysis,
experiments, visualisation, and writing the original draft. D.B. contributed to
validation, interpretation of results, and critical review and revision of the
manuscript. Both authors read and approved the final manuscript and accept
responsibility for the work.

\subsection*{Data Availability}

All experiments use publicly available benchmark data. The sources and dataset
identifiers are listed in Table~\ref{tab:data-sources}. No new data were
collected for this study.

\subsection*{Code Availability}

Code and reproducibility materials are available at
\url{https://github.com/jopoku16/sef-scopegate}. The replication archive
includes a single runner that regenerates the main result tables and figures,
as described in the Reproducibility section.

\subsection*{Ethics Approval}

This study uses public benchmark data and does not report any new
human-subject or animal data collection. Ethics approval was therefore not
required.

\subsection*{Consent to Participate and Consent to Publish}

Not applicable. The study involves no human participants and no individual
personal data.

\bibliographystyle{plainnat}
\bibliography{references}

@article{opoku2026spectral,
  title={Spectral Adaptive Conformal Prediction for Structured Non-Exchangeable Data},
  author={Opoku, Jeffery and Banahene, David},
  journal={arXiv preprint arXiv:2606.15950},
  year={2026},
  url={https://arxiv.org/abs/2606.15950}
}

@article{opoku2026dasc,
  title={Drift-Aware Spectral Conformal Prediction for Non-Exchangeable Streaming Data},
  author={Opoku, Jeffery and Banahene, David},
  journal={arXiv preprint arXiv:2606.15953},
  year={2026},
  url={https://arxiv.org/abs/2606.15953}
}

@article{opoku2026promptshift,
  title={{PromptShift-CRC}: Drift-Aware Conformal Risk Control for Foundation Models Under Prompt and Domain Shift},
  author={Opoku, Jeffery and Banahene, David},
  journal={arXiv preprint arXiv:2606.15964},
  year={2026},
  url={https://arxiv.org/abs/2606.15964}
}

@article{opoku2026toolchain,
  title={{ToolChain-CRC}: Conformal Risk Control for Agentic {AI} Under Retrieval and Tool-Use Drift},
  author={Opoku, Jeffery and Banahene, David},
  journal={arXiv preprint arXiv:2606.18467},
  year={2026},
  url={https://arxiv.org/abs/2606.18467}
}

@article{shapley1953value,
  title={A Value for n-Person Games},
  author={Shapley, Lloyd S.},
  journal={Contributions to the Theory of Games},
  volume={2},
  pages={307--317},
  year={1953}
}

@inproceedings{lundberg2017unified,
  title={A Unified Approach to Interpreting Model Predictions},
  author={Lundberg, Scott M. and Lee, Su-In},
  booktitle={Advances in Neural Information Processing Systems},
  year={2017}
}

@inproceedings{ribeiro2016why,
  title={Why Should {I} Trust You? Explaining the Predictions of Any Classifier},
  author={Ribeiro, Marco Tulio and Singh, Sameer and Guestrin, Carlos},
  booktitle={Proceedings of the ACM SIGKDD International Conference on Knowledge Discovery and Data Mining},
  pages={1135--1144},
  year={2016}
}

@inproceedings{sundararajan2017axiomatic,
  title={Axiomatic Attribution for Deep Networks},
  author={Sundararajan, Mukund and Taly, Ankur and Yan, Qiqi},
  booktitle={International Conference on Machine Learning},
  pages={3319--3328},
  year={2017}
}

@article{meinshausen2010stability,
  title={Stability Selection},
  author={Meinshausen, Nicolai and B{\"u}hlmann, Peter},
  journal={Journal of the Royal Statistical Society: Series B},
  volume={72},
  number={4},
  pages={417--473},
  year={2010}
}

@article{shafer2008tutorial,
  title={A Tutorial on Conformal Prediction},
  author={Shafer, Glenn and Vovk, Vladimir},
  journal={Journal of Machine Learning Research},
  volume={9},
  pages={371--421},
  year={2008}
}

@article{angelopoulos2021gentle,
  title={A Gentle Introduction to Conformal Prediction and Distribution-Free Uncertainty Quantification},
  author={Angelopoulos, Anastasios N. and Bates, Stephen},
  journal={arXiv preprint arXiv:2107.07511},
  year={2021}
}

@article{bates2021distribution,
  title={Distribution-Free, Risk-Controlling Prediction Sets},
  author={Bates, Stephen and Angelopoulos, Anastasios N. and Lei, Lihua and Malik, Jitendra and Jordan, Michael I.},
  journal={Journal of the ACM},
  volume={68},
  number={6},
  pages={1--34},
  year={2021}
}

@article{schapire1998boosting,
  title={Boosting the Margin: A New Explanation for the Effectiveness of Voting Methods},
  author={Schapire, Robert E. and Freund, Yoav and Bartlett, Peter and Lee, Wee Sun},
  journal={The Annals of Statistics},
  volume={26},
  number={5},
  pages={1651--1686},
  year={1998}
}

@article{murdoch2019definitions,
  title={Definitions, Methods, and Applications in Interpretable Machine Learning},
  author={Murdoch, W. James and Singh, Chandan and Kumbier, Karl and Abbasi-Asl, Reza and Yu, Bin},
  journal={Proceedings of the National Academy of Sciences},
  volume={116},
  number={44},
  pages={22071--22080},
  year={2019}
}

@article{breiman2001random,
  title={Random Forests},
  author={Breiman, Leo},
  journal={Machine Learning},
  volume={45},
  pages={5--32},
  year={2001}
}

@article{friedman2001greedy,
  title={Greedy Function Approximation: A Gradient Boosting Machine},
  author={Friedman, Jerome H.},
  journal={The Annals of Statistics},
  volume={29},
  number={5},
  pages={1189--1232},
  year={2001}
}

@article{doshi2017rigorous,
  title={Towards a Rigorous Science of Interpretable Machine Learning},
  author={Doshi-Velez, Finale and Kim, Been},
  journal={arXiv preprint arXiv:1702.08608},
  year={2017}
}

@article{rudin2019stop,
  title={Stop Explaining Black Box Machine Learning Models for High Stakes Decisions and Use Interpretable Models Instead},
  author={Rudin, Cynthia},
  journal={Nature Machine Intelligence},
  volume={1},
  pages={206--215},
  year={2019}
}

@article{fisher2019all,
  title={All Models are Wrong, but Many are Useful: Learning a Variable's Importance by Studying an Entire Class of Prediction Models Simultaneously},
  author={Fisher, Aaron and Rudin, Cynthia and Dominici, Francesca},
  journal={Journal of Machine Learning Research},
  volume={20},
  number={177},
  pages={1--81},
  year={2019}
}

@article{apley2020visualizing,
  title={Visualizing the Effects of Predictor Variables in Black Box Supervised Learning Models},
  author={Apley, Daniel W. and Zhu, Jingyu},
  journal={Journal of the Royal Statistical Society: Series B},
  volume={82},
  number={4},
  pages={1059--1086},
  year={2020}
}

@article{geifman2017selective,
  title={Selective Classification for Deep Neural Networks},
  author={Geifman, Yonatan and El-Yaniv, Ran},
  journal={Advances in Neural Information Processing Systems},
  year={2017}
}

@article{geifman2019selectivenet,
  title={{SelectiveNet}: A Deep Neural Network with an Integrated Reject Option},
  author={Geifman, Yonatan and El-Yaniv, Ran},
  journal={International Conference on Machine Learning},
  year={2019}
}

@misc{blackard1998covtype,
  title={Covertype Data Set},
  author={Blackard, Jock A. and Dean, Denis J. and Anderson, Charles W.},
  howpublished={UCI Machine Learning Repository},
  year={1998}
}

@book{boucheron2013concentration,
  title={Concentration Inequalities: A Nonasymptotic Theory of Independence},
  author={Boucheron, Stephane and Lugosi, Gabor and Massart, Pascal},
  publisher={Oxford University Press},
  year={2013}
}

\end{document}